%% file: Efficient_online_learning_with_Kernels_for_adversarial_large_scale_problems__preprint.tex
\documentclass{article}

\usepackage{arxiv}

\usepackage[utf8]{inputenc} 
\usepackage[T1]{fontenc}    
\usepackage{hyperref}       
\usepackage{microtype}
\usepackage{graphicx}
\usepackage{subfigure}
\usepackage{booktabs} 
\usepackage{amsfonts}
\usepackage{amsmath,amssymb,amsthm} 
\usepackage{enumitem}
\usepackage{dsfont}
\usepackage[numbers]{natbib}
\usepackage{algorithm}
\usepackage{caption}
\usepackage{xcolor}
\usepackage{pgfplots}
\pgfplotsset{compat=1.8}
\usepackage{tikzscale}
\usepackage{adjustbox}
\usepackage{ifthen}
\usepackage{cleveref}
\usepackage{listings}
\usepackage{multirow}
\usepackage{cancel}

\usepackage{url}            
\usepackage{booktabs}       
\usepackage{nicefrac}       

\usepackage{cleveref}
\usepackage{listings}

\newtheorem{theorem}{Theorem}
\newtheorem{proposition}[theorem]{Proposition}
\newtheorem{lemma}[theorem]{Lemma}
\newtheorem{corollary}[theorem]{Corollary}
\newtheorem{assumption}{Assumption}
\newtheorem{definition}{Definition}[section]
\newtheorem{example}{Example}[section]
\theoremstyle{remark}
\newtheorem{remark}{Remark}[section]

\newcommand{\tphi}{{\tilde{\phi}}}
\newcommand{\kawv}{Kernel-AWV}
\newcommand{\pkawv}{PKAWV}
\newcommand{\deff}{d_{\textrm{eff}}}

\newcommand{\eqals}[1]{\begin{align*}#1\end{align*}}
\newcommand{\eqal}[1]{\begin{align}#1\end{align}}
\newcommand{\bpr}{\begin{proof}}
\newcommand{\epr}{\end{proof}}
\newcommand{\be}{\begin{equation}}
\newcommand{\ee}{\end{equation}}

\newcommand{\bd}{\begin{definition}}
\newcommand{\ed}{\end{definition}}

\newcommand{\bi}{\begin{itemize}}
\newcommand{\ei}{\end{itemize}}

\newcommand{\ba}{\begin{assumption}}
\newcommand{\ea}{\end{assumption}}

\newcommand{\bre}{\begin{restatable}}
\newcommand{\ere}{\end{restatable}}

\newcommand{\br}{\begin{remark}}
\newcommand{\er}{\end{remark}}

\newcommand{\bp}{\begin{proposition}}
\newcommand{\ep}{\end{proposition}}

\newcommand{\blm}{\begin{lemma}}
\newcommand{\elm}{\end{lemma}}

\newcommand{\bt}{\begin{theorem}}
\newcommand{\et}{\end{theorem}}

\newcommand{\bcor}{\begin{corollary}}
\newcommand{\ecor}{\end{corollary}}

\newcommand{\bex}{\begin{example}}
\newcommand{\eex}{\end{example}}

\newcommand{\cI}{\mathcal{I}}
\newcommand{\cH}{\mathcal{H}}
\newcommand{\cX}{\mathcal{X}}
\renewcommand{\hat}{\widehat}
\newcommand{\tCn}{\widetilde{C}_n}
\newcommand{\Kn}{K_{nn}}

\renewcommand{\ge}{\geq}
\renewcommand{\le}{\leq}
\newcommand{\indic}{\mathds{1}}
\crefname{assumption}{Assumption}{Assumptions}
\crefname{equation}{Eq.}{Eqs.}
\crefname{figure}{Fig.}{Figs.}
\crefname{table}{Table}{Tables}
\crefname{section}{Sec.}{Secs.}
\crefname{theorem}{Thm.}{Thms.}
\crefname{proposition}{Prop.}{Props.}
\crefname{lemma}{Lemma}{Lemmas}
\crefname{corollary}{Cor.}{Cors.}
\crefname{example}{Example}{Examples}
\crefname{appendix}{Appendix}{Appendixes}
\crefname{remark}{Remark}{Remark}

\input{macro}

\def\mystrut(#1,#2){\vrule height #1pt depth #2pt width 0pt}

\title{Efficient online learning with Kernels for adversarial large scale problems}

\author{
  Rémi Jézéquel 
   \And
  Pierre Gaillard
  \And
  Alessandro Rudi \\
}

\date{INRIA - Département d’Informatique de l’École Normale Supérieure \\
PSL Research University \\
Paris, France}

\begin{document}
\maketitle

\begin{abstract}
We are interested in a framework of online learning with kernels for low-dimensional but large-scale and potentially adversarial datasets. 
We study the computational and theoretical performance of online variations of kernel Ridge regression. Despite its simplicity, the algorithm we study is the first  to achieve the optimal regret for a wide range of kernels with a per-round complexity of order $n^\alpha$ with $\alpha < 2$. 

The algorithm we consider is based on approximating the kernel with the linear span of basis functions. Our contributions is two-fold: 1) For the Gaussian kernel, we propose to build the basis beforehand (independently of the data) through Taylor expansion. For $d$-dimensional inputs, we provide a (close to) optimal regret of order $O((\log n)^{d+1})$ with per-round time complexity and space complexity $O((\log n)^{2d})$. This makes the algorithm a suitable choice as soon as $n \gg e^d$ which is likely to happen in a scenario with small dimensional and large-scale dataset; 2) For general kernels with low effective dimension, the basis functions are updated sequentially in a data-adaptive fashion by sampling Nyström points. In this case, our algorithm improves the computational trade-off known for online kernel regression.

\end{abstract}

\section{Introduction}

Nowadays the volume and the velocity of data flows are deeply increasing. Consequently many applications need to switch from batch to online procedures that can treat and adapt to data on the fly. Furthermore to take advantage of very large datasets, non-parametric methods are gaining increasing momentum in practice. Yet the latter often suffer from slow rates of convergence and bad computational complexities. At the same time, data is getting more complicated and simple stochastic assumptions such as i.i.d. data are often not satisfied. 
In this paper, we try to combine these different aspects due to large scale and arbitrary data. We  build a non-parametric online procedure based on kernels, which is efficient for large data sets and achieves close to optimal theoretical guarantees.

Online learning is a subfield of machine learning where some learner sequentially interacts with an environment and tries to learn and adapt on the fly to the observed data as one goes along. We consider the following sequential setting. At each iteration $t\geq 1$, the learner receives some input $x_t \in \cX$; makes a prediction $\hat y_t \in \R$ and the environment reveals the output $y_t \in \R$. The inputs $x_t$ and the outputs $y_t$ are sequentially chosen by the environment and can be arbitrary. 
Learner's goal is to minimize his cumulative regret 
\begin{equation}
    \label{eq:defRegret}
   { R_n(f) := \sum_{t=1}^n (y_t-\hat y_t)^2 - \sum_{t=1}^n \big(y_t - f(x_t)\big)^2}
\end{equation}
uniformly over all functions $f$ in a space of functions $\cH$. 
We will consider Reproducing Kernel Hilbert Space (RKHS) $\cH$,  \citep[see next section or ][for more details]{aronszajn1950theory}. It is worth noting here that all the properties of a RKHS are controlled by the associated {\em kernel function} $k:\X \times \X \to \R$, usually known in closed form, and that many function spaces of interest are (or are contained in) RKHS, e.g. when $\X \subseteq \R^d$: polynomials of arbitrary degree, band-limited functions, analytic functions with given decay at infinity, Sobolev spaces and many others \citep{berlinet2011reproducing}.


\paragraph{Previous work} Kernel regression in a statistical setting has been widely studied by the statistics
community. Our setting of online kernel regression with adversarial data is more recent. 
Most of existing work focuses on the linear setting (i.e., linear kernel). First work on online linear regression dates back to \cite{Foster1991}. \cite{Bartlett2015} provided the minimax rates (together with an algorithm) and we refer to reader to references therein for a recent overview of the literature in the linear case. We only recall relevant work for this paper.  \cite{AzouryWarmuth2001,Vovk01} designed the nonlinear Ridge forecaster (denoted {\em AWV}). 
In linear regression (linear kernel), it achieves the optimal regret of order $O(d \log n)$ uniformly over all $\ell_2$-bounded vectors. The latter can be extended to kernels (see Definition~\eqref{eq:nonLinearRidge}) which we refer to as \kawv{}. With regularization parameter $\lambda>0$, it  obtains a regret upper-bounded for all $f \in \cH$ as
\begin{equation}
    R_n(f) \lesssim  \lambda  \nor{f}^2 + B^2 \deff(\lambda)\,, \quad
    \label{eq:optimalRegret}
\text{where} \quad 
    \deff(\lambda) := \tr(K_{nn}\big(K_{nn} + \lambda I_n)^{-1}\big)
\end{equation}
is the effective dimension, where $K_{nn} := {\big(k(x_i,x_j)\big)_{1\leq i,j\leq n} \in \R^{n\times n}}$ denotes the {\em kernel matrix} at time $n$. The above upper-bound on the regret is essentially optimal (see next section). Yet the per round complexity and the space complexity of \kawv{}  are  $O(n^2)$. In this paper, we aims at reducing this complexity  while keeping optimal regret guarantees.

Though the literature on online contextual learning is vast, little considers non-parametric function classes. Related work includes \cite{Vov-06-MetricEntropyOnlinePrediction} that considers the Exponentially Weighted Average forecaster or \cite{HaMe-07colt-OnlineLearningPriorKnowledge} which considers bounded Lipschitz function set and Lipschitz loss functions, while here we focus on the square loss. Minimax rates for general function sets $\cH$ are provided by \cite{RaSrTs-13-MinimaxRegretRisk}. RKHS spaces were first considered in \cite{Vov-05-OnlineRegressionRKHS} though they only obtain $O(\sqrt{n})$ rates which are suboptimal for our problem. More recently, a regret bound of the form \eqref{eq:optimalRegret} was proved by \cite{zhdanov2010identity} for a clipped version of kernel Ridge regression and by \cite{calandriello2017second} for a clipped version of Kernel Online Newton Step (KONS) for general exp-concave loss functions. 

The computational complexity ($O(n^2)$ per round) of these algorithms is however prohibitive for large datasets.  \cite{calandriello2017second} and \cite{calandriello2017efficient} provide approximations of KONS to get manageable complexities. However these come with deteriorated regret guarantees. \cite{calandriello2017second} improves the time and space complexities by a factor $\gamma \in (0,1)$ enlarging the regret upper-bound by $1/\gamma$. \cite{calandriello2017efficient} designs an efficient approximation of KONS based on Nyström approximation \citep{smola2000sparse,williams2001using} and restarts with per-round complexities $O\big(m^2)$ where $m$ is the number of Nyström points. Yet their regret bound suffers an additional multiplicative factor $m$ with respect to~\eqref{eq:optimalRegret} because of the restarts. Furthermore, contrary to our results, the regret bounds of \cite{calandriello2017second} and \cite{calandriello2017efficient} are not with respect to all functions in $\cH$ but only with functions $f \in \cH$ such that $f(x_t) \leq C$ for all $t\geq 1$  where $C$ is a parameter of their algorithm. Since $C$ comes has a multiplicative factor of their bounds, their results are sensitive to outliers that may lead to large $C$.  
Another relevant approximation scheme of Online Kernel Learning was done by~\cite{lu2016large}. The authors consider online gradient descent algorithms which they approximate using Nyström approximation or a Fourier basis. However since they use general Lipschitz loss functions and consider $\ell_1$-bounded dual norm of functions $f$, their regret bounds of order $O(\sqrt{n})$ are hardly comparable to ours and seem suboptimal in $n$ in our restrictive setting with square loss and kernels with small effective dimension (such as Gaussian kernel).

\paragraph{Contributions and outline of the paper}{}
The main contribution of the paper is to analyse a variant of \kawv{} that we call \pkawv{} (see Definition~\eqref{eq:pkawv-def}). Despite its simplicity, it is to our knowledge the first algorithm for kernel online regression that recovers the optimal regret (see bound~\eqref{eq:optimalRegret}) with an improved space and time complexity of order $\ll n^2$ per round. Table~\ref{tab:rates-summary} summarizes the regret rates and complexities obtained by our algorithm and the ones of~\cite{calandriello2017second,calandriello2017efficient}.

Our procedure consists simply in applying \kawv{} while, at time $t\geq 1$, approximating the RKHS $\cH$ with a linear subspace $\smash{\tilde \cH_t}$ of smaller dimension. In Theorem~\ref{thm:pkawv-appr-error}, \pkawv{} suffers an additional approximation term with respect to the optimal bound of~\kawv{} which can be made small enough by properly choosing $\smash{\tilde \cH_t}$. To achieve the optimal regret with a low computational complexity, $\smash{\tilde \cH_t}$ needs to approximate $\cH$ well and to be low dimensional with an easy-to-compute projection. We provide two relevant constructions for $\smash{\tilde \cH_t}$. 

In section~\ref{sec:taylor}, we focus on the Gaussian kernel that we approximate by a finite set of basis functions. The functions are deterministic and chosen beforehand by the player independently of the data. The number of functions included in the basis is a parameter to be optimized and fixes an approximation-computational trade-off.
Theorem~\ref{thm:phi-awv-gaussian} shows that \pkawv{} satisfies (up to log) the optimal regret bounds~\eqref{eq:optimalRegret} while enjoying a per-round space and time complexity of $\smash{O\big(\log^{2d}\left(\frac{n}{\la}\right)\big)}$. For the Gaussian kernel, this corresponds to $\smash{O\big(\deff(\lambda)^2\big)}$ which is known to be optimal even in the statistical setting with i.i.d. data. 

In section~\ref{sec:nystrom}, we consider data adaptive approximation spaces $\smash{\tilde \cH_t}$ based on Nyström approximation. At time $t\geq 1$, we approximate any kernel $\cH$ by sampling a subset of the input vectors $\{x_1,\dots,x_t\}$. If the kernel satisfies the capacity condition $\smash{\deff(\lambda)\leq (n/\lambda)^\gamma}$ for $\gamma \in (0,1)$, the optimal regret is then of order $\smash{\deff(\lambda)} = \smash{O(n^{\gamma/(1+\gamma)})}$ for well-tuned parameter $\lambda$. Our method then recovers the optimal regret with a computational complexity of $\smash{O\big(\deff(\lambda)^{4/(1-\gamma)}\big)}$. The latter is $o(n^2)$ (for well-tuned $\lambda$) as soon as $\smash{\gamma < \sqrt{2}-1}$. Furthermore, if the sequence of input vectors $x_t$ is given beforehand to the algorithm, the per-round complexity needed to reach the optimal regret is improved to $\smash{O(\deff(\lambda)^{4})}$ and our algorithm can achieve it for all $\gamma \in (0,1)$.

Finally, we perform in Section~\ref{sec:experiments} several experiments based on real and simulated data to compare the performance (in regret and in time) of our methods with competitors. 

\begin{table}
    \centering
    \begin{tabular}{clcc} \toprule
        Kernel & Algorithm & Regret & Per-round complexity \\ \midrule
        \multirow{3}{*}{\shortstack{Gaussian \\ $\deff(\lambda) \leq \big(\log \frac{n}{\lambda}\big)^d$}} & \pkawv{} & $(\log n)^{d+1}$ & $(\log n)^{2d}$ \\
         & Sketched-KONS~\cite{calandriello2017second} ($c>0$)& $c(\log n)^{d+1}$ & $\big(n/c\big)^2$ \\
        & Pros-N-KONS~\cite{calandriello2017efficient} & $(\log n)^{2d+1}$ & $(\log n)^{2d}$ \\ \midrule
        \multirow{3}{*}{\shortstack{General \\ $\deff(\lambda) \leq \big(\frac{n}{\lambda}\big)^\gamma$ \\$\gamma < \sqrt{2}-1$}}  & \pkawv{} & $n^{\frac{\gamma}{\gamma+1}} \log n$ & $n^{\frac{4\gamma}{1-\gamma^2}}$ \\
         & Sketched-KONS~\cite{calandriello2017second} ($c>0$) & $c n^{\frac{\gamma}{\gamma+1}} \log n$ & $\big(n/c\big)^2$ \\
         & Pros-N-KONS~\cite{calandriello2017efficient} & $n^{\frac{4\gamma}{(1+\gamma)^2}} \log n$  & $n^{\frac{4\gamma(1-\gamma)}{(1+\gamma)^2}}$ \\ \bottomrule \\
    \end{tabular}
    \caption{Order in $n$ of the best possible regret rates achievable by the algorithms and corresponding per-round time-complexity. Up to $\log n$, the rates obtained by~\pkawv{} are optimal.}
    \label{tab:rates-summary}
\end{table}
\paragraph{Notations} We recall here basic notations that we will use throughout the paper. Given a vector $v \in \R^d$, we write $v = (v^{(1)},\dots,v^{(d)})$. We denote by $\N_0 = \N \cup \{0\}$ the set of non-negative integers and for $p \in \N_0^d$, $|p| = p^{(1)} + \dots + p^{(d)}$.  By a sligh abuse of notation, we denote by $\|\cdot\|$ both the Euclidean norm and the norm for the Hilbert space $\cH$. Write $v^\top w$, the dot product between $v, w \in \R^D$. The conjugate transpose for linear operator $Z$ on $\cH$ will be denoted $Z^*$. The notation $\lesssim$ will refer to rough inequalities up to logarithmic multiplicative factors. Finally we will denote $a \vee b = \max(a,b)$ and $a \wedge b = \min(a,b)$, for $a,b \in \R$.

\section{Background}
\label{sec:background}

\paragraph{Kernels.}
Let $k:\X \times \X \to \R$ be a positive definite kernel \citep{aronszajn1950theory} that we assume to be bounded (i.e., $\sup_{x \in \X} k(x,x) \leq \kappa^2$ for some $\kappa>0$). The function $k$ is characterized by the existence of a {\em feature map} $\phi:\X \to \R^D$, with $D \in \N \cup \{\infty\}$\footnote{when $D = \infty$ we consider $\R^D$ as the space of squared summable sequences.} such that 
$k(x,x') = \phi(x)^\top \phi(x')$. Moreover the {\em reproducing kernel Hilbert space} (RKHS) associated to $k$ is characterized by $\hh = \{f ~|~ f(x) = w^\top \phi(x), ~w \in \R^D, x \in \X\},$ with inner product $\scal{f}{g}_\hh := v^\top w$, for $f, g \in \hh$ defined by $f(x) = v^\top \phi(x)$, $g(x) = w^\top \phi(x)$ and $v,w \in \R^D$. For more details and different characterizations of $k, \hh$, see \citep{aronszajn1950theory,berlinet2011reproducing}.
It's worth noting that the knowledge of $\phi$ is not necessary when working with functions of the form $f = \sum_{i=1}^p \alpha_i \phi(x_i)$, with $\alpha_i \in \R$, $x_i \in \X$ and finite $p \in \N$, indeed 
$
    f(x) = \sum_{i=1}^p \alpha_i \phi(x_i)^\top \phi(x) = \sum_{i=1}^p \alpha_i k(x_i, x),
$
and moreover $\|f\|_\hh^2 = \alpha^\top K_{pp} \alpha$, with $K_{pp}$ the kernel matrix associated to the set of points  $x_1,\dots,x_p$.
 
\paragraph{\kawv{}.}  The  (denoted {\em AWV}) on the space of linear functions on $\X = \R^d$ has been introduced and analyzed in \cite{AzouryWarmuth2001,Vovk01}. 
We consider here a straightforward generalization to kernels (denoted {\em \kawv{}}) of the nonlinear Ridge forecaster ({\em AWV}) introduced by  \cite{AzouryWarmuth2001,Vovk01} on the space of linear functions on $\cX = \R^d$. At iteration $t\geq 1$, \kawv{} predicts $\hat{y}_t = \hat{f}_t(x_t)$, where
\begin{equation}
    \label{eq:nonLinearRidge}
    \hat f_t \in \argmin{f \in \cH} \left\{  \sum_{s=1}^{t-1} \big(y_s - f(x_s)\big)^2 + \lambda \nor{f}^2 + f(x_t)^2 \right\} .
\end{equation}
A variant of this algorithm, more used in the context of data independently sampled from distribution, is known as {\em kernel Ridge regression}. It corresponds to solving the problem above, without the last penalization term $f(x_t)^2$.

 \noindent{\bf Optimal regret for \kawv.}
 In the next proposition we state a preliminary result which proves that \kawv{} achieves a regret depending on the eigenvalues of the kernel matrix. 
\begin{proposition}
    \label{prop:nonLinearRidge}
    Let $\lambda, B > 0$. For any RKHS $\cH$,  for all $n \geq 1$, for all inputs  $x_1,\dots,x_n \in \cX$ and all $y_1,\dots,y_n \in [-B,B]$,  the regret of \kawv{}  is upper-bounded for all $f \in \cH$ as 
    \eqals{
       R_n(f)  \leq \lambda \nor{f}^2 + B^2 \sum_{k=1}^n \log \left(1 + \frac{\lambda_k(K_{nn})}{\lambda}\right)\,, }
where  $\lambda_k(K_{nn})$ denotes the $k$-th largest eigenvalue of $K_{nn}$.
\end{proposition}
 The proof is a direct consequence of the known regret bound of {\em AWV} in the finite dimensional linear regression setting---see Theorem 11.8 of \cite{Cesa-Bianchi2006} or Theorem~2 of \cite{gaillard2018uniform}. For completeness, we reproduce the analysis for infinite dimensional space (RKHS) in Appendix~\ref{app:kernelnonlinearRidge}. In online linear regression in dimension $d$, the above result implies the optimal rate of convergence $ dB^2\log(n) + O(1)$ (see~\cite{gaillard2018uniform} and \cite{Vovk01}). As shown by the following proposition, Proposition~\ref{prop:nonLinearRidge} yields optimal regret (up to log) of the form~\eqref{eq:optimalRegret} for online kernel  regression. 
 
 \begin{proposition}
 \label{prop:upperbounddeff}
 For all $n\geq 1$, $\lambda >0$ and all input sequences $x_1,\dots,x_n \in \cX$,
 \[
 \sum_{k=1}^n \log \left(1 + \frac{\lambda_k(K_n)}{\lambda}\right) \leq \log \Big(e + \frac{e n \kappa^2}{\lambda}\Big) \deff\big(\lambda\big) \,.
 \]
 \end{proposition}
 
Combined with Proposition~\ref{prop:nonLinearRidge}, this entails that \kawv{} satisfies (up to the logarithmic factor) the optimal regret bound~\eqref{eq:optimalRegret}. As discussed in the introduction, such an upper-bound on the regret is not new and was already proved by~\cite{zhdanov2010identity} or by \cite{calandriello2017second} for other algorithms. An advantage of \kawv{} is that it does not require any clipping and thus the beforehand knowledge of $B>0$ to obtained Proposition~\ref{prop:nonLinearRidge}. Furthermore, we slightly improve the constants in the above proposition. We note that the regret bound for \kawv{} is optimal up to log terms when $\la$ is chosen minimizing r.h.s. of \cref{eq:optimalRegret}, since it meets known minimax lower bounds for the setting where $(x^{(i)}, y_i)_{i=1}^n$ are sampled independently from a distribution. For more details on minimax lower bounds see \citep{zhang2015divide}, in particular Eq.~(9), related discussion and references therein, noting that their $\la$ correspond to our $\la/n$ and our $\deff(\la)$ corresponds to their $\gamma(\la/n)$.

It is worth pointing out that in the worst case $\deff(\lambda) \leq \kappa^2 n /\lambda$ for any bounded kernel. In particular, optimizing the bound yields $\lambda = O(\sqrt{n \log n})$ and a regret bound of order $O(\sqrt{n \log n})$. In the special case of the Gaussian kernel (which we consider in Section~\ref{sec:taylor}), the latter can be improved to $\smash{\deff(\lambda) \lesssim \big(\log(n/\lambda)\big)^d}$ (see~\cite{altschuler2018massively}) which entails $R_n(f) \leq O\big((\log n)^{d+1}\big)$ for well tuned value of $\lambda$.

\section{Online Kernel Regression with projections}
\label{sec:kernelproj}

In previous section we have seen that \kawv{} achieves optimal regret. Yet, it has computational requirements that are $O(n^3)$ in time and $O(n^2)$ in space, for $n$ steps of the algorithm, making it unfeasible in the context of large scale datasets, i.e. $n \gg 10^5$. In this paper, we consider and analyze a simple variation of \kawv{} denoted \pkawv. At time $t\geq 1$, for a regularization parameter $\lambda >0$ and a linear subspace $\smash{\tilde \cH_t}$ of $\cH$ the algorithm predicts $\smash{\hat{y}_t = \hat{f}_t(x_t)}$, where
\begin{equation}
    \label{eq:pkawv-def}
    \hat f_t = \argmin{f \in \tilde \cH_t} \left\{ \sum_{s=1}^{t-1} \big(y_s - f(x_s)\big)^2 + \lambda \nor{f}^2 + f(x_t)^2 \right\}\,.
\end{equation}  
In the next subsections, we explicit relevant approximations $\tilde \cH_t$ (typically the span of a small number of basis functions) of $\cH$  that trade-off good approximation with low computational cost. Appendix~\ref{app:implementation} details how \eqref{eq:pkawv-def} can be efficiently implemented  in these cases.

The result below bounds the regret of the \pkawv{} for any function $f \in \hh$ and holds for any bounded kernel and any explicit subspace $\smash{\tilde \cH}$ associated with projection $P$. The cost of the approximation of $\cH$ by $\smash{\tilde \cH}$ is measured by the important quantity $\smash{\mu := \big\|(I-P) C_n^{1/2}\big\|^2}$, where $C_n$ is the covariance operator. 

\begin{theorem}\label{thm:pkawv-appr-error} Let $\tilde \cH$ be a linear subspace of $\cH$ and $P$ the Euclidean projection onto $\tilde \cH$.
When \pkawv{} is run with $\lambda >0$ and fixed subspaces $\tilde \cH_t = \tilde \cH$, then for all $f \in \cH$ 
\begin{equation}\label{eq:regret-appr-pawv}
        R_n(f)  \leq \lambda \nor{f}^2 + B^2 \sum_{j=1}^n  \log\left(1+ \frac{\la_j(\Kn)}{\la} \right) 
        + (\mu + \lambda)\frac{n\mu B^2}{\lambda^2} \,,
\end{equation}
for any sequence $(x_1,y_1),\dots,(x_n,y_n) \in \cX \times [-B,B]$ where $\mu :=  \big\|(I-P) C_n^{1/2}\big\|^2$ and $C_n :=\sum_{t=1}^n \phi(x_t) \otimes \phi(x_t)$.
\end{theorem}

The proof of \cref{thm:pkawv-appr-error} is deferred to \cref{sec:proofthm_pkawv-appr-error} and is the consequence of a more general \cref{thm:mainLemma}.


\subsection{Learning with Taylor expansions and Gaussian kernel for very large data set}
\label{sec:taylor}

 In this section we focus on non-parametric regression with the widely used {\em Gaussian kernel} defined by 
$
k(x,x') = \exp(-\|x-x'\|^2/(2\sigma^2))
$
for $x,x'\in \cX$ and $\sigma > 0$ and the associated RKHS $\hh$. 

Using the results of the previous section with a fixed linear subspace $\smash{\tilde \cH}$ which is the span of a basis of $O(\textrm{polylog}(n/\la))$ functions, we prove that \pkawv{} achieves optimal regret. This leads to a computational complexity that is only $O(n~ \textrm{polylog}(n/\la))$ for optimal regret. We need a basis that (1) approximates very well the Gaussian kernel and at the same time (2) whose projection is easy to compute. We consider the following basis of functions, for $k \in \N_0^d$,
\eqal{ \label{eq:basis}
g_k(x) = \prod_{i=1}^d \psi_{k_i}(x^{(i)}), \quad \text{where} \quad \psi_t(x) = \frac{x^t}{\sigma^{t} \sqrt{t!}}e^{-\frac{x^2}{2\sigma^2}}. 
}
For one dimensional data this corresponds to Taylor expansion of the Gaussian kernel.
Our theorem below states that \pkawv{} (see~\eqref{eq:pkawv-def}) using for all iterations $t\geq 1$  \\[4pt]
{\centering $ \hspace*{2cm} \tilde \cH_t = \mathrm{Span}(G_M) \qquad \text{with}  \quad G_M = \{ g_k ~|~ |k| \leq M, k \in \N_0^d\}$} \\[4pt]
where $|k| := k_1 + \dots + k_d$, for $k \in \N_0^d$, 
gets optimal regret while enjoying low complexity. The size of the basis $M$ controls the trade-off between approximating well the Gaussian kernel (to incur low regret) and large computational cost. Theorem~\ref{thm:phi-awv-gaussian} optimizes $M$ so that the approximation term of Theorem~\ref{thm:pkawv-appr-error} (due to kernel  approximation) is of the same order than the optimal regret. 

\bt\label{thm:phi-awv-gaussian}
Let $\la > 0, n \in \N$ and let $R, B > 0$. Assume that $\|x_t\| \leq R$ and $|y_t| \leq B$. When
$M = \big\lceil \frac{8R^2}{\sigma^2} ~\vee~ 2\log\frac{n}{\la \wedge 1} \big\rceil$, then running \pkawv{} using $G_M$ as set of functions  achieves a regret bounded by
    $$
    \smash{\forall f\in \cH, R_n(f) \leq \lambda \nor{f}^2 + \frac{3B^2}{2} \sum_{j=1}^n \log\left(1 + \frac{\la_j(K_{nn})}{\lambda}\right). } \mystrut(10,10)
    $$
Moreover, its per iteration computational cost is
    $O\big(\big(3 + \frac{1}{d}\log\frac{n}{\la \wedge 1}\big)^{2d}\big)$
in space and time.
\et

Therefore \pkawv{} achieves a regret-bound only deteriorated by a multiplicative factor of $3/2$ with respect to the bound obtained by \kawv{} (see \cref{prop:nonLinearRidge}). From~\cref{prop:upperbounddeff} this also yields (up to log) the optimal bound~\eqref{eq:optimalRegret}.  

In particular, it is known \citep{altschuler2018massively}  for the Gaussian kernel that
\[
    \deff(\lambda) \leq 3 \Big(6+\frac{41}{d} \frac{R^2}{2\sigma^2} + \frac{3}{d} \log \frac{n}{\lambda}\Big)^d = O \Big(\Big(\log \frac{n}{\lambda}\Big)^d\Big) \,.
\]
The upper-bound is matching even in the i.i.d. setting for non trivial distributions. In this case, we have $|G_M| \lesssim \deff(\lambda)$. The per-round space and time complexities are thus $O\big(\deff(\lambda)^2\big)$. Though our method is quite simple (since it uses fixed explicit embedding) it is able to recover results -in terms of computational time and bounds in the adversarial setting- that are similar to results obtained in the more restrictive i.i.d. setting obtained via much more sophisticated methods, like learning with (1) Nystr\"om with importance sampling via leverage scores \citep{rudi2015less}, (2) reweighted random features \citep{bach2017equivalence,rudi2017generalization}, (3) volume sampling \cite{derezinski2018correcting}. 
By choosing $\smash{\la  = (B/\|f\|)^2},$
to minimize the r.h.s. of the regret bound of the theorem, we get 
\[
    \smash{R_n(f) \lesssim \Big(\log \frac{n \|f\|^2_\hh}{B^2}\Big)^{d+1} B^2.} \mystrut(0,10)
\]
Note that the optimal $\lambda$ does not depend on $n$ and can be optimized in practice through standard online calibration methods such as using  an expert advice algorithm \citep{Cesa-Bianchi2006} on a finite grid of $\lambda$. Similarly, though we use a fixed number of features $M$ in the experiments, the latter could be increased slowly over time thanks to online calibration techniques. 


\subsection{Nyström projection}
\label{sec:nystrom}

The previous two subsections considered deterministic basis of functions (independent of the data) to approximate specific RKHS. 
Here, we analyse Nyström projections \cite{rudi2015less} that are data dependent and works for any RKHS. It consists in sequentially updating a dictionary $\cI_t \subset \{x_1,\dots,x_t\}$ and using 
\begin{equation}
    \label{eq:nystromsubspace}
    \tilde \cH_t = \mathrm{Span}\Big\{ \phi(x),\  x \in \cI_t \,\Big\}\,.
\end{equation}
If the points included into $\cI_t$ are well-chosen, the latter may approximate well the solution of \eqref{eq:nonLinearRidge} which belongs to the linear span of $\{\phi(x_1),\dots,\phi(x_t)\}$.
The inputs $x_t$ might be included into the dictionary independently and uniformly at random. Here, we build the dictionary by following the KORS algorithm of \cite{calandriello2017efficient}  which is based on approximate leverage scores. At time $t\geq 1$, it evaluates the importance of including $x_t$ to obtain an accurate projection $P_t$ by computing its leverage score. Then, it decides to add it or not by drawing a Bernoulli random variable. The points are never dropped from the dictionary so that $\cI_1\subset\cI_2\subset \cdots \cI_n$. With their notations, choosing $\varepsilon =1/2$ and remarking that  $\smash{\|\Phi_t^T(I-P_t)\Phi_t\| 
= \|(I-P_t)C_t^{1/2}\|^2}$, their Proposition 1 can be rewritten as follows.

\begin{proposition}{\cite[Prop. 1]{calandriello2017efficient}}
\label{thm:KORS}
Let $\delta > 0$, $n\geq 1$, $\mu >0$. Then, the sequence of dictionaries $\cI_1\subset \cI_2 \subset \cdots \subset \cI_n$ learned by KORS with parameters $\mu$ and $\beta =12 \log(n/\delta)$ satisfies w.p. $1-\delta$, 
\[ 
   \smash{\forall t\geq 1, \qquad \big\|(I-P_t)C_t^{1/2}\big\|^2 \leq \mu \qquad 
    \text{and}
    \qquad 
    |\cI_t| \leq 9  \deff(\mu) \log \big(2n/\delta\big)^2\,.} \mystrut(6,6)
\]
Furthermore, the algorithm runs in $O\big(\deff(\mu)^2 \log^4(n)\big)$ space and $O\big(\deff(\mu)^2\big)$ time per iteration.
\end{proposition}

Using this approximation result together with \cref{thm:mainLemma} (which is a more general version of \cref{thm:pkawv-appr-error}), we can bound the regret of \pkawv{} with KORS. The proof is postponed to Appendix~\ref{sec:proof_pkawv-nystrom}.

\begin{theorem}
\label{thm:pkawv-nystrom}
Let $n \geq 1$, $\delta >0$ and $\lambda \geq \mu >0$. Assume that the dictionaries $(\cI_t)_{t\geq 1}$ are built according to Proposition~\ref{thm:KORS}. Then,  probability at least $1-\delta$, \pkawv{} with the subspaces $\tilde \cH_t$ defined in~\eqref{eq:nystromsubspace} satisfies the regret upper-bound
\[
    \smash{R_n \leq \lambda \|f\|^2 + B^2\deff(\lambda)\log \big( e + en \kappa^2/\lambda \big) + 2B^2 (|\cI_n|+1)\frac{n \mu}{\lambda}\,,} \mystrut(10,5)
\]
and the algorithm runs in $O(\deff(\mu)^2)$ space $O(\deff(\mu)^2)$ time per iteration. 
\end{theorem}

The last term of the regret upper-bound above corresponds to the approximation cost of using the approximation~\eqref{eq:nystromsubspace} in \pkawv{}. This costs is controlled by the parameter $\mu>0$ which trades-off between having a small approximation error (small $\mu$) and a small dictionary of size $|\cI_n| \approx \deff(\mu)$ (large $\mu$) and thus a small computational complexity. For the Gaussian Kernel, using that $\deff(\lambda) \leq O\big( \log (n/\lambda)^d\big)$, the above theorem yields for the choice $\lambda = 1$ and $\mu = n^{-2}$ a regret bound of order $R_n \leq O\big((\log n)^{d+1}\big)$ with a per-round time and space complexity of order $O(|\cI_n|^2) = O\big((\log n)^{2d+4}\big)$. We recover a similar result to the one obtained in Section~\ref{sec:taylor}.

\paragraph{Explicit rates under the capacity condition} Assuming the capacity condition  $\smash{\deff(\lambda')} \le \smash{\big( n/\lambda'\big)^\gamma}$ for  $0 \leq \gamma \leq 1$ and $\lambda'>0$, which is a classical assumption made on kernels \cite{rudi2015less}, the following corollary provides explicit rates for the regret according to the size of the dictionary $m \approx |\cI_n|$.

\begin{figure}
\centering
\begin{minipage}{.3\textwidth}
  \begin{center}
  \def\gam{0.35}
  \def\plotlegend{T}
  \begin{tikzpicture}[scale=.5]
  \input{rates.tex}
  \end{tikzpicture}
  \end{center}
 \end{minipage}
\begin{minipage}{.3\textwidth}
  \begin{center}
  \def\gam{0.41}
  \def\plotlegend{F}
  \begin{tikzpicture}[scale=.5]
  \input{rates.tex}
  \end{tikzpicture}
  \end{center}
 \end{minipage}
 \begin{minipage}{.3\textwidth}
  \begin{center}
  \def\gam{0.75}
  \def\plotlegend{F}
  \begin{tikzpicture}[scale=.5]
  \input{rates.tex}
  \end{tikzpicture}
  \end{center}
 \end{minipage}
\caption{Comparison of the theoretical regret rate $R_n = O(n^b)$ according to the size of the dictionary $m = O(n^a)$ considered by  \pkawv{}, Sketched-KONS and Pros-N-KONS for optimized parameters when $\smash{\deff(\lambda) \leq (n/\lambda)^\gamma}$ with $\gamma = 0.25, \sqrt{2}-1, 0.75$ (from left to right).}
\label{fig:rates}
\end{figure}
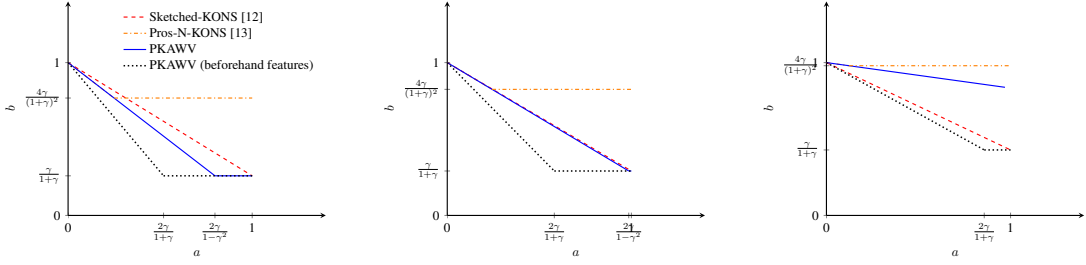

\begin{corollary}
\label{cor:krr_nystrom_rate}
Let $n \geq 1$ and $m \geq 1$. Assume that $\deff(\lambda') \leq (n/\lambda')^\gamma$ for all $\lambda' >0$. Then, under the assumptions of Theorem~\ref{thm:pkawv-nystrom}, \pkawv{} with $\mu = nm^{-1/\gamma}$ has a dictionary of size $\smash{|\cI_n| \lesssim m}$ and a regret upper-bounded with high-probability as
\begin{equation*}
        R_n \lesssim \left\{ \begin{array}{lll}
            n^{\frac{\gamma}{1+\gamma}} & \text{if } m \geq n^{\frac{2\gamma}{1-\gamma^2}} & \text{for } \lambda = n^\frac{\gamma}{1+\gamma} \\
            nm^{\frac{1}{2}-\frac{1}{2\gamma}} & \text{otherwise} & \text{for } \lambda = nm^{\frac{1}{2}-\frac{1}{2\gamma}}
        \end{array}\right. \,.
\end{equation*}
The per-round space and time complexity of the algorithm is $O(m^2)$ per iteration. 
\end{corollary}

The rate of order $n^{\frac{\gamma}{1+\gamma}}$ is optimal in this case (it corresponds to optimizing~\eqref{eq:optimalRegret} in $\lambda$). If the dictionary is large enough $\smash{m \geq n^{2\gamma/(1-\gamma^2)}}$, the approximation term is negligible and the algorithm recovers the optimal rate. This is possible for a small dictionary $m = o(n)$  whenever $\smash{2\gamma/(1-\gamma^2) <1}$, which corresponds to $\smash{\gamma < \sqrt{2}-1}$. The rates obtained in Corollary~\ref{cor:krr_nystrom_rate} can be compared to the one obtained by Sketched-KONS of~\cite{calandriello2017second} and Pros-N-KONS of \cite{calandriello2017efficient} which also provide a similar trade-off between the dictionary size $m$ and a regret bound. The forms of the regret bounds in $m$, $\mu$, $\lambda$ of the algorithms can be summarized as follow
\begin{equation}
    R_n  \lesssim \left\{ 
        \begin{array}{ll}
            \lambda + \deff(\lambda) + \frac{nm\mu}{\lambda}  & \text{for \pkawv{} with KORS} \\ 
            \lambda  + \frac{n}{m} \deff(\lambda)  & \text{for Sketched-KONS} \\
            m(\lambda + \deff(\lambda)) + \frac{n\mu}{\lambda} & \text{for Pros-N-KONS}
        \end{array}\right. \,.
        \label{eq:rates_comparison}
\end{equation}
When $\deff(\lambda) \leq (n/\lambda)^\gamma$, optimizing these bounds in $\lambda$, \pkawv{} performs better than Sketched-KONS as soon as $\gamma \leq 1/2$ and the latter cannot obtain the optimal rate $\smash{\lambda + \deff(\lambda) = n^{\frac{\gamma}{1+\gamma}}}$ if $m = o(n)$. Furthermore, because of the multiplicative factor $m$, Pros-N-KONS can't either reached the optimal rate even for $m=n$.  Figure~\ref{fig:rates} plots the rate in $n$ of the regret of these algorithms when enlarging the size $m$ of the dictionary. We can see that for $\gamma =1/4$, \pkawv{} is the only algorithm that achieves the optimal rate $n^{\gamma/(1+\gamma)}$ with $m = o(n)$ features. The rate of Pros-N-KONS cannot beat $4\gamma/(1+\gamma)^2$ and stops improving even when the size of dictionary increases. This is because Pros-N-KONS is restarted whenever a point is added in the dictionary which is too costly for large dictionaries. It is worth pointing out that these rates are for well-tuned value of $\lambda$. However, such an optimization can be performed at small cost using expert advice algorithm on a finite grid of~$\lambda$.

\paragraph{Beforehand known features} We may assume that the sequence of feature vectors $x_t$ is given in advance to the learner while only the outputs $y_t$ are sequentially revealed (see \cite{gaillard2018uniform} or \cite{Bartlett2015} for details). In this case, the complete dictionary $\cI_n \subset \{x_1,\dots,x_n\}$ may be computed beforehand and \pkawv{} can be used with the fix subspace $\tilde \cH = \mathrm{Span}(\phi(x),x\in \cI_n)$. 
In this case, the regret upper-bound can be improved to $R_n \lesssim \lambda + \deff(\lambda) + \frac{n\mu}{\lambda}$ by removing a factor $m$ in the last term (see~\eqref{eq:rates_comparison}). 

\begin{corollary}
\label{cor:krr_nystrom_known_features_rate}
Under the notation and assumptions of Corollary~\ref{cor:krr_nystrom_rate}, \pkawv{} used with dictionary $\cI_n$ and parameter $\mu = nm^{-1/\gamma}$ achieves with high probability
\begin{equation*}
        R_n \lesssim \left\{ \begin{array}{lll}
            n^{\frac{\gamma}{1+\gamma}} & \text{if } m \geq n^{\frac{2\gamma}{1+\gamma}} & \text{for } \lambda = n^\frac{\gamma}{1+\gamma} \\
            nm^{-\frac{1}{2\gamma}} & \text{otherwise} & \text{for } \lambda = nm^{-\frac{1}{2\gamma}}
        \end{array}\right. \,.
\end{equation*}
Furthermore, w.h.p. the dictionary is of size $|\cI_n| \lesssim m$ leading to a per-round space and time complexity $O(m^2)$. 

\end{corollary}

The suboptimal rate due to a small dictionary is improved by a factor $\sqrt{m}$ compared to the ``sequentially revealed features'' setting. Furthermore, since $2\gamma/(1+\gamma) < 1$ for all $\gamma \in (0,1)$, the algorithm is able to recover the optimal rate $n^{\gamma/(1+\gamma)}$ for all $\gamma \in (0,1)$ with a dictionary of sub-linear size $m \ll n$. We leave for future work the question whether there is really a gap between these two settings or if this gap from a suboptimality of our analysis.

\begin{figure*}[t]
\centering
\includegraphics[scale=.18]{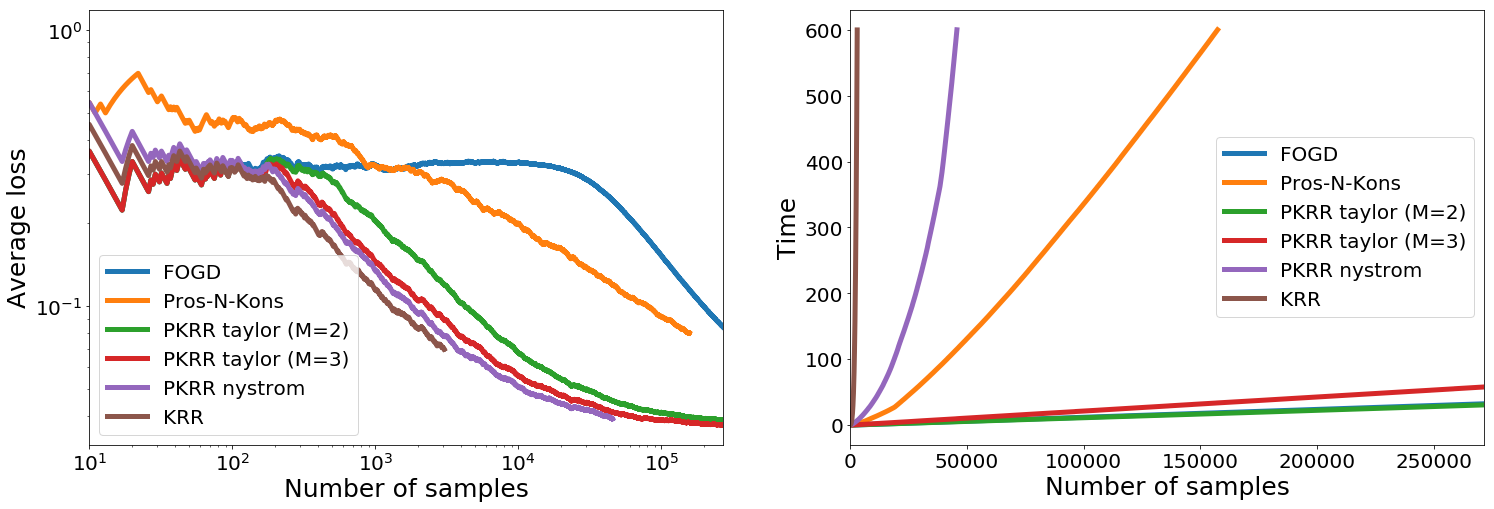}
\includegraphics[scale=.18]{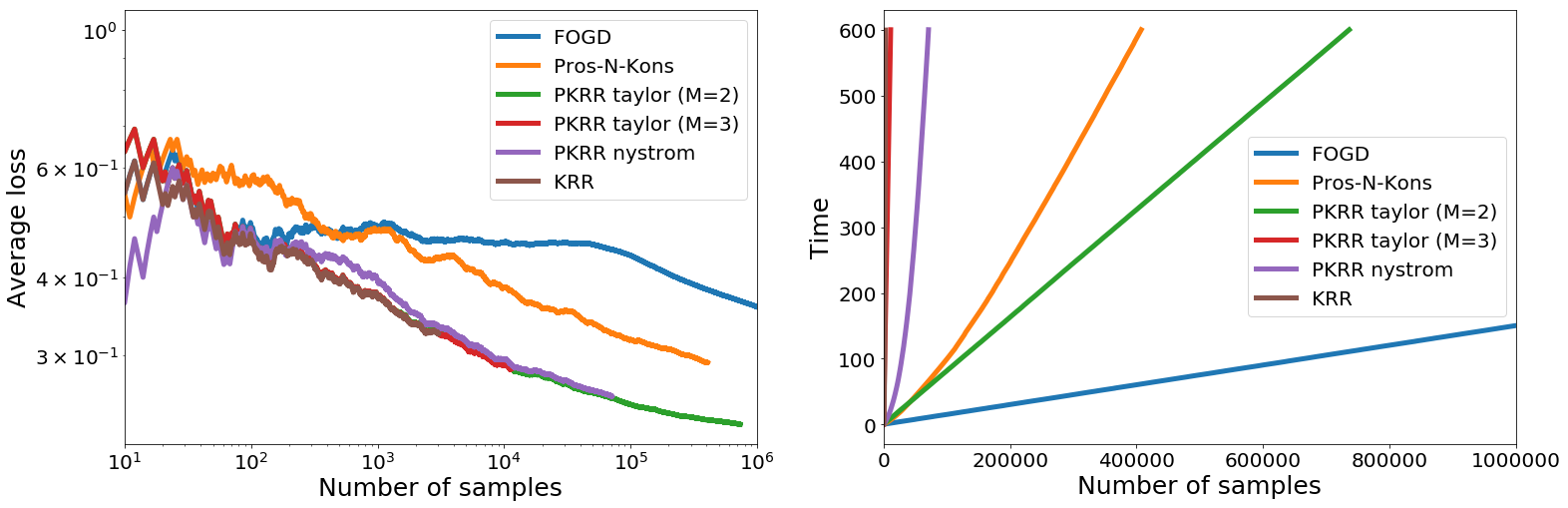}
\caption{Average classification error and time on: (top) code-rna ($n=2.7\times10^5,~~d=8$); (bottom) SUSY ($n=6\times10^6,~~d=22$).}
\label{susy_perf}
\end{figure*}

\section{Experiments}
\label{sec:experiments}
We empirically test \pkawv{} against several state-of-the-art algorithms for online kernel regression. In particular we test our algorithms in (1) an adversarial setting [see Appendix \ref{sec:exp-adv}], (2) on large scale datasets. The following algorithms have been tested:
\begin{itemize}[topsep=-3pt,nosep,nolistsep]
    \item Kernel-AWV for adversarial setting or Kernel Ridge Regression for i.i.d. real data settings;
    \item Pros-N-Kons \cite{calandriello2017second};
    \item  Fourier Online Gradient Descent (FOGD, \cite{lu2016large});
    \item \pkawv (or Projected-KRR for real data settings) with Taylor expansions  ($M \in \{2,3,4\})$
    \item \pkawv (or Projected-KRR for real data settings) with Nyström
\end{itemize}
The algorithms above have been implemented in python with numpy (the code for our algorithm is in \cref{app:python-code}). For most algorithms we used hyperparameters from the respective papers. For all algorithms and all experiments, we set $\sigma = 1$ \citep[except for SUSY where $\sigma = 4$, to match accuracy results from][]{rudi2017falkon} and $\lambda = 1$. When using KORS, we set $\mu = 1$, $\beta = 1$ and $\varepsilon = 0.5$ as in \cite{calandriello2017second}. The number of random-features in FOGD is fixed to $1000$ and the learning rate $\eta$ is $\smash{1/\sqrt{n}}$. All experiments have been done on a single desktop computer (Intel Core i7-6700) with a timeout of $5$-min per algorithm. The results of the algorithms are only recorded until this time.

\noindent{\bf Large scale datasets.}
 The algorithms are evaluated on four datasets from UCI machine learning repository. In particular \texttt{casp} (regression) and \texttt{ijcnn1}, \texttt{cod-rna}, \texttt{SUSY} (classification) [see Appendix \ref{sec:other_dataset} for \texttt{casp} and \texttt{ijcnn1}] ranging from $4\times10^4$ to $6\times10^6$ datapoints. For all datasets, we scaled $x$ in $[-1,1]^d$ and $y$ in $[-1,1]$. In  \cref{casp_perf,susy_perf} we show the average loss (square loss for regression and classification error for classification) and the computational costs of the considered algorithm.

%


In all the experiments \pkawv{} with $M = 2$ approximates reasonably well the performance of kernel forecaster and is usually very fast. We remark that using \pkawv{} $M=2$ on the first million examples of \texttt{SUSY}, we achieve in $10$ minutes on a single desktop, the same average classification error obtained with specific large scale methods for i.i.d. data \citep{rudi2017falkon}, although \kawv{} is using a number of features reduced by a factor $100$ with respect to the one used in for FALKON in the same paper. Indeed they used $r = 10^4$ Nystr\"om centers, while with $M=2$ we used $r=190$ features, validating empirically the effectiveness of the chosen features for the Gaussian kernel. This shows the effectiveness of the proposed approach for large scale machine learning problems with moderate dimension $d$.



\bibliography{biblio}
\bibliographystyle{unsrt}  

\newpage
\hrule
\begin{center}
{\huge Supplementary material}
\end{center}
\hrule

\appendix

The supplementary material is organized as follows:
\begin{itemize}
    \item  Appendix~\ref{app:definitions} starts with notations and useful identities that are used in the rest of the proofs
    \item Appendix~\ref{app:mainLemma},~\ref{app:proofsbackground},~\ref{app:proofTaylor},~\ref{app:proof_Nystrom}, and~\ref{app:additional_proofs}  contain the proofs mostly in order of appearance:
    \begin{itemize}
        \item Appendix~\ref{app:mainLemma}: statement and proof of our main theorem on which are based most of our results.
        \item Appendix~\ref{app:proofsbackground}: proofs of Section~\ref{sec:background} (Propositions~\ref{prop:nonLinearRidge} and \ref{prop:upperbounddeff})
        \item Appendix~\ref{app:proofTaylor}: proofs of section~\ref{sec:taylor} (Theorem~\ref{thm:pkawv-appr-error} and~\ref{thm:phi-awv-gaussian})
        \item Appendix~\ref{app:proof_Nystrom}: proofs of section~\ref{sec:nystrom} (Theorem~\ref{cor:krr_nystrom_rate} and Corollaries~\ref{cor:krr_nystrom_rate} and \ref{cor:krr_nystrom_known_features_rate})
        \item Appendix~\ref{app:additional_proofs}: proofs of additional lemmas.
    \end{itemize} 
    \item Appendix~\ref{app:other_experiments} provides additional experimental results (adversarial simulated data and large-scale real datasets).
    \item Appendix~\ref{app:implementation} describes efficient implementations of our algorithms together with the Python code used for the experiments.
\end{itemize}

\section{Notations and relevant equations}
\label{app:definitions}

In this section, we give notations and useful identities which will be used in following proofs. We recall that at time $t \geq 1$, the forecaster is given an input $x_t \in \cX \subset \R^d$, chooses a prediction function $\hat f_t \in \smash{\tilde \cH_t} \subset \cH$,  forecasts $\smash{\hat y_t = \hat f_t(x_t)}$ and observes $\hat y_t \in [-B,B]$. Moreover, $\cH$ is the RKHS associated to the kernel $k:(x,x') \in \cX \times \cX = \phi(x)^\top \phi(x')$ for some feature map $\phi:\cX \to \R^D$. We also define the following notations for all $t\geq 1$:

\begin{itemize}[nosep, label={--}]
	\item $Y_t = (y_1,\dots,y_t)^\top\in \R^t$ and $\hat Y_t = (\hat y_1,\dots,\hat y_t)^\top \in \R^t$
	\item $P_t:\cH \to \tilde \cH_t$ is the Euclidean projection on $\tilde \cH_{t}$
	\item $C_t := \sum_{i=1}^t \phi(x_i) \otimes \phi(x_i)$ is the covariance operator at time $t\geq 1$; 
	\item $A_t := C_t + \lambda I$ is the regularized covariance operator;
	\item $S_t:\cH\to \R^t$ is the operator such that $[S_tf]_i = f(x_i) = \scal{f}{\phi(x_i)}$ for any $f \in \cH$; 
	\item $L_t:= f \in \cH  \mapsto  \big\|Y_t - S_t f\big\|^2 + \lambda \|f\|^2$ is the regularized cumulative loss. 
\end{itemize}

The prediction function of \pkawv{} at time $t \geq 1$ is defined (see Definition~\ref{eq:pkawv-def}) as
\begin{equation*}
	\hat f_t = \arg\min_{f\in {\tilde \cH_t}}\left\{\sum_{s=1}^{t-1}\big(y_s - f(x_s)\big)^2 + \lambda \|f\|^2 + f(x_t)^2 \right\} \,. 
\end{equation*}
Standard calculation shows the equality 
\begin{equation}
    \label{eq:f_def}
    \hat f_t = P_t \tilde A_t^{-1} P_t S^*_{t-1} Y_{t-1} \,.
\end{equation}

We define also the best functions in the subspace $\tilde \cH_t$ and $\tilde \cH_{t+1}$ at time $t\geq 1$,
\begin{equation}
	\hat g_{t+1} = \arg\min_{f\in \tilde \cH_{t}}\left\{ L_{t}(f)\right\}= P_t \tilde A_{t}^{-1} P_t S^*_{t} Y_{t} \,,
	\label{eq:g_def}
\end{equation}
\begin{equation}
	\tilde g_{t+1} = \arg\min_{f\in \tilde \cH_{t+1}}\left\{ L_{t}(f)\right\}= P_{t+1} (P_{t+1} C_t P_{t+1} + \lambda I)^{-1} P_{t+1} S^*_{t} Y_{t} \,,
	\label{eq:g_tilde_def}
\end{equation}
and the best function in the whole space $\cH$
\begin{equation}
	\hat h_{t+1} = \arg\min_{f\in \cH}\left\{ L_{t}(f)\right\}= A_t^{-1} S^*_{t} Y_{t} \,.
	\label{eq:h_def}
\end{equation}

\section{Main theorem (statement and proof)}
\label{app:mainLemma}

In this appendix, we provide a general upper-bound on the regret of \pkawv{} that is valid for any sequence of projections $P_1,...,P_n$ associated with the sequence $\smash{\tilde \cH_1},\dots,\smash{\tilde \cH_n}$. Many of our results will be corollaries of the following theorem for specific sequences of projections.

\begin{theorem}
\label{thm:mainLemma}
Let $\tilde \cH_1,\dots,\tilde \cH_n$ be a sequence of linear subspaces of $\cH$ associated with projections $P_1,\dots,P_n \in \cH \to \cH$. \pkawv{} with regularization parameter $\lambda >0$ satisfies the following upper-bound on the regret: for all $f \in \cH$
\[
	R_n(f) \leq \sum\limits_{t=1}^n y_t^2 \scal{\tilde A_t^{-1} P_t \phi(x_t)}{P_t \phi(x_t)}
    	+ (\mu_{t} + \lambda) \frac{\mu_{t} t B^2}{\lambda}  \,,
\]
for any sequence $(x_1,y_1),\dots,(x_n,y_n) \in \cX \times [-B,B]$ and where $\mu_{t} := \big\|(P_{t+1} - P_t) C_{t}^{1/2}\big\|^2$ and $P_{n+1} := I$.
\end{theorem}

\begin{proof}
Let $f\in \cH$. By definition of $\hat h_{n+1}$ (see~\eqref{eq:h_def}), we have $L_{n}(\hat h_{n+1}) \leq L_n(f)$ which implies by definition of $L_n$ that
\begin{equation}
	\big\|Y_n - S_n \hat h_{n+1}\big\|^2 - \big\|Y_n - S_n f \big\|^2 \leq \lambda \|f\|^2 - \lambda \|\hat h_{n+1}\|^2  \,.
	\label{eq:firsteq}
\end{equation}
Now, the regret can be upper-bounded as 
\begin{align}
	R_n(f) & \stackrel{\eqref{eq:defRegret}}{:=}  \sum_{t=1}^n (y_t-\hat y_t)^2 - \sum_{t=1}^n \big(y_t - f(x_t)\big)^2 \\
	    &  = \big\|Y_n - \hat Y_n\big\|^2 - \big\|Y_n - S_nf\big\|^2  \nonumber \\
		& \stackrel{\eqref{eq:firsteq}}{\leq}   \big\|Y_n - \hat Y_n\big\|^2 - \big\|Y_n - S_n \hat h_{n+1}\big\|^2 + \lambda \|f\|^2 - \lambda \|\hat h_{n+1}\|^2 \nonumber \\
		& \leq  \lambda \|f\|^2 +  \underbrace{\big\|Y_n - \hat Y_n\big\|^2 - \big\|Y_n - S_n \hat g_{n+1}\big\|^2 - \lambda \|\hat g_{n+1} \|^2}_{Z_1} \label{eq:defZ}\\
		& \hspace{1cm} + \underbrace{\big\|Y_n - S_n \hat g_{n+1}\big\|^2 + \lambda \|\hat g_{n+1}\|^2  - \big\|Y_n - S_n \hat h_{n+1}\big\|^2 - \lambda \|\hat h_{n+1}\|^2}_{\Omega(n+1)}  \nonumber 
\end{align}
The first term $Z_1$ mainly corresponds to the estimation error of the algorithm: the regret incurred with respect to the best function in the approximation space $\smash{\tilde \cH_n}$. It also includes an approximation error due to the fact that the algorithm does not use $\smash{\tilde \cH_n}$ but the sequence of approximation $\smash{\tilde \cH_1},\dots, \smash{\tilde \cH_n}$. The second term $\Omega(n+1)$ corresponds to the approximation error of $\cH$ by $\smash{\tilde \cH_n}$.  Our analysis will focus on upper-bounding both of these terms separately.

\paragraph{Part 1. Upper-bound of the estimation error $Z_1$.} Using a telescoping argument together with the convention $L_0(\hat g_1) = 0$, we have
\[
    \big\|Y_n - S_n \hat g_{n+1}\big\|^2 +  \lambda \|\hat g_{n+1} \|^2 = L_n(\hat g_{n+1}) = \sum_{t=1}^n L_{t}(\hat g_{t+1}) - L_{t-1}(\hat g_t)\,.
\]
Substituted into the definition of $Z_1$ (see~\eqref{eq:defZ}), the latter can be rewritten as
\begin{align}
    Z_1 &= \sum_{t=1}^n \big[(y_t-\hat y_t)^2 +  L_{t-1}(\hat g_{t}) - L_{t}(\hat g_{t+1})\big] \nonumber \\
    &= \sum_{t=1}^n \big[\underbrace{(y_t-\hat y_t)^2 +  L_{t-1}(\tilde g_{t})- L_t(\hat g_{t+1})}_{Z(t)} + \underbrace{L_{t-1}(\hat g_{t}) - L_{t-1}(\tilde g_{t})}_{\Omega(t)} \big] \,.
		\label{eq:regret1}
\end{align}

where $\tilde g_t = P_t(P_t C_{t-1} P_t + \lambda I)^{-1} P_t S_{t-1}^* Y_{t-1}$ is obtained by substituting $P_t$ with $P_{t-1}$ in the definition of $\hat g_t$. Note that with the convention $P_{n+1} = I$  the second term $\Omega(t)$ matches the definition of $\Omega(n+1)$  of~\eqref{eq:defZ} since $\smash{\tilde g_{n+1} = A_n^{-1} S^*_n Y_n = \hat h_{n+1}}$. In the rest of the first part we focus on upper-bounding the terms $Z(t)$. The approximation terms $\Omega(t)$ will be bounded in the next part.

Now, we remark that by expanding the square norm
\begin{multline}
	\label{eq:Lt}
	L_t(f) = \|Y_t\|^2 - 2Y_t^\top S_t f  + \big\|S_t f\big\|^2 + \lambda \|f\|^2  =   \|Y_t\|^2 - 2Y_t^\top S_t f  + \scal{f}{C_t f} + \lambda \|f\|^2 \\
	= \|Y_t\|^2 - 2Y_t^\top S_t f  + \scal{f}{A_t f} \,,
\end{multline}
where for the second equality, we used 
\[
\big\|S_t f\big\|^2  = \sum_{t=1}^n f(x_t)^2 = \sum_{t=1}^n \scal{f}{\phi(x_t)}^2 = \sum_{t=1}^n \scal{f}{\phi(x_t) \otimes \phi(x_t) f} = \scal{f}{C_tf}\,.
\]
Substituting $\hat g_{t+1}$ into~\eqref{eq:Lt} we get
\begin{equation}
\label{eq:Ltgt}
L_t(\hat g_{t+1})  = \|Y_t\|^2 - 2Y_t^\top S_t \hat g_{t+1} + \scal{\hat g_{t+1}}{A_t \hat g_{t+1}}\,.
\end{equation}
But, since $\hat g_{t+1} \in \tilde H_{t}$, we have $\hat g_{t+1} =  P_t \hat g_{t+1}$ which yields 
\[
Y_t^\top S_t \hat g_{t+1} = Y_t^\top S_t  P_t \hat g_{t+1} = Y_t^\top S_t  \tilde A_t^{-1} \tilde A_t P_t \hat g_{t+1} \,.
\]
Then, using that $\tilde A_t P_t = (P_t C_t P_t + \lambda I)P_t = P_t  A_t P_t$, we get
\[
 Y_t^\top S_t \hat g_{t+1} = \underbrace{Y_t^\top S_t P_t \tilde A_t^{-1} P_t}_{\hat g_{t+1}^\top}   A_t  \hat g_{t+1} =  \scal{\hat g_{t+1}}{A_t \hat g_{t+1}} \,.
\]
Thus, combining with~\eqref{eq:Ltgt} we get
\[
    L_t(\hat  g_{t+1})  = \|Y_t\|^2 - \scal{\hat g_{t+1}}{A_t \hat g_{t+1}} \,.
\]
Similarly, substituting $\tilde g_t$ into~\eqref{eq:Lt} and using $\tilde g_t \in \tilde \cH_t$, we can show
\[
    L_{t-1}(\tilde g_t) = \|Y_{t-1}\|^2 - \scal{\tilde  g_{t}}{A_{t-1} \tilde g_{t}} \,.
\]
Combining the last two equations implies
\begin{equation}
	L_{t-1}(\tilde g_{t}) - L_t(\hat g_{t+1}) = - y_t^2 + \scal{\hat g_{t+1}}{A_t \hat g_{t+1}} - \scal{\tilde g_{t}}{A_{t-1} \tilde g_{t}}\,.
	\label{eq:instantdiff2}
\end{equation}
Furthermore, using the definition of $\hat g_{t+1}$, we have
\[
    P_t A_t \hat g_{t+1} = P_t A_t P_t \tilde A_t^{-1} P_t S_t^*Y_t = P_t \tilde A_t \tilde A_{t}^{-1} P_t S_t^*Y_t = P_t S_t^* Y_t \,.
\]
The same calculation with $\tilde g_t$ yields
\begin{multline}
    P_t A_{t-1}\tilde g_t = P_t (C_{t-1} + \lambda I) P_t  (P_t C_{t-1}P_t + \lambda I)^{-1} S_{t-1}^* Y_{t-1}  \\
    = P_t (P_t C_{t-1}P_t + \lambda I)  (P_t C_{t-1}P_t + \lambda I)^{-1} S_{t-1}^* Y_{t-1} = P_t S_{t-1}^* Y_{t-1}\,.
    \label{eq:ptattildegt}
\end{multline}
Together with the previous equality, it entails
\begin{equation}
    P_t A_t \hat g_{t+1} - P_t A_{t-1}\tilde g_t  =   P_t(S_t^* Y_t - S_{t-1}^* Y_{t-1})  = y_t P_t \phi(x_t) \,. \label{eq:AtKxt2}
\end{equation}

Then, because $\hat f_t \in \tilde \cH_t$, we have 
 $\hat y_t = \hat f_t(x_t) = \scal{\hat f_t}{\phi(x_t)} = \scal{\hat f_t}{P_t \phi(x_t)}$. This yields
\begin{align}
	(y_t - \hat y_t)^2 
		& = y_t^2 - 2 y_t \hat y_t + \hat y_t^2 \nonumber \\
		& = y_t^2 - 2 \scal{\hat f_t}{y_t P_t \phi(x_t)} + \scal{\hat f_t}{ \phi(x_t)\otimes \phi(x_t) \hat f_t} \nonumber \\
		& \stackrel{\eqref{eq:AtKxt2}}{\leq} y_t^2 - 2\scal{\hat f_t}{P_t A_t \hat g_{t+1} - P_t A_{t-1} \tilde g_t} + \scal{\hat f_t}{ \phi(x_t)\otimes \phi(x_t) \hat f_t}   \nonumber \\
		& =  y_t^2 - 2\scal{\hat f_t}{A_t \hat g_{t+1} - A_{t-1} \tilde g_t} + \scal{\hat f_t}{ (A_t - A_{t-1}) \hat f_t}   \,,
		\label{eq:instantloss2}
\end{align}
where the last equality uses $f_t \in \tilde \cH_t$ and that $A_t - A_{t-1} = \phi(x_t) \otimes \phi(x_t)$. 

Putting equations~\eqref{eq:instantdiff2} and \eqref{eq:instantloss2} together,
we get 
\begin{eqnarray*}
    Z(t)
	& \stackrel{\eqref{eq:regret1}}{=} & (y_t - \hat y_t)^2  + L_{t-1}(\tilde g_{t}) - L_t(\hat g_{t+1})  \\
	&\stackrel{\eqref{eq:instantdiff2}+\eqref{eq:instantloss2}}{\leq} & \Big(\scal{\hat g_{t+1}}{A_t \hat g_{t+1}} - 2 \scal{\hat f_t}{A_t\hat g_{t+1}} + \scal{\hat f_t}{A_t\hat  f_{t}} \Big) \\
	& & \hspace*{1cm} - \Big(\scal{\tilde g_{t}}{A_{t-1} \hat g_{t}} - 2 \scal{P_{t-1} \hat f_t}{A_{t-1} \tilde g_{t}} + \scal{\hat f_t}{A_{t-1}\hat  f_{t}}  \Big)\, \\
	& = & 	\scal{\hat g_{t+1} - \hat f_t}{A_t(\hat g_{t+1} - \hat f_t)} - \underbrace{\scal{\hat f_t - \tilde g_t}{A_{t-1}(\hat f_t - \tilde g_t)}}_{\geq 0}  \\
	& \leq &  \scal{\hat g_{t+1} - \hat f_t}{\tilde A_t(\hat g_{t+1} - \hat f_t)} \\
	& \stackrel{\eqref{eq:f_def}+\eqref{eq:g_def}}{=} &  \scal{P_t \tilde A_t^{-1} P_t (S_t^* Y_t - S_{t-1}^* Y_{t-1})}{\tilde A_t P_t \tilde A_t^{-1} P_t (S_t^* Y_t - S_{t-1}^* Y_{t-1})} \\
	& =  &  y_t^2 \scal{P_t \tilde A_t^{-1} P_t \phi(x_t)}{\tilde A_t P_t \tilde A_t^{-1} P_t \phi(x_t)} \\
	& = & y_t^2 \scal{\tilde A_t^{-1} P_t \phi(x_t)}{P_t \phi(x_t)} \\
\end{eqnarray*}
where the last equality is because $P_t \tilde A_t = \tilde A_t P_t$ from the definition of $\tilde A_t := \tilde C_t + \lambda I$ with $\tilde C_t := P_t C_t P_t$.

Therefore, plugging back into~\eqref{eq:regret1}, we have
\begin{equation}
    Z_1 \leq \sum_{t=1}^n y_t^2 \scal{\tilde A_t^{-1} P_t \phi(x_t)}{P_t \phi(x_t)}  + \Omega(t) \,,
    \label{eq:part1}
\end{equation}
where we recall that $\Omega(t) := L_{t-1}(\hat g_t) - L_{t-1}(\tilde g_t)$.

\textbf{Part 2. Upper-bound of the approximation terms $\Omega(t)$.}
We recall that we use the convention $P_{n+1} = I$ which does not change the algorithm. Let $t \geq 1$, expending the square losses we get
\begin{align*}
    \Omega(t+1) &= \sum\limits_{s=1}^t \left[ (\hat g_{t+1}(x_s) - y_s)^2 - (\tilde g_{t+1}(x_s) - y_s)^2  + \lambda \|\hat g_{t+1}\|^2 - \lambda \|\tilde g_{t+1}\|^2  \right]\\
    &= \sum\limits_{s=1}^t \big[ \cancel{y_s^2} - 2 \scal{\hat g_{t+1}}{y_s \phi(x_s)} + \scal{\hat g_{t+1}}{\phi(x_s) \otimes \phi(x_s) \hat g_{t+1}} \\
    &\hspace*{1cm}  - \cancel{y_s^2} + 2 \scal{\tilde g_{t+1}}{y_s \phi(x_s)} - \scal{\tilde g_{t+1}}{\phi(x_s) \otimes \phi(x_s) \tilde g_{t+1}} + \lambda \|\hat g_{t+1}\|^2 - \lambda \|\tilde g_{t+1}\|^2 \big]  \\
    &= 2 \scal{\tilde g_{t+1} - \hat g_{t+1}}{S_t^* Y_t} + \scal{\hat g_{t+1}}{A_t \hat g_{t+1}} - \scal{\tilde g_{t+1}}{A_t \tilde g_{t+1}}  
\end{align*}
Since both $\tilde g_{t+1}$ and $\hat g_{t+1}$ belong to $\tilde \cH_{t+1}$, we have
\[
    \Omega(t+1) = 2 \scal{\tilde g_{t+1} - \hat g_{t+1}}{P_{t+1} S_t^* Y_t} + \scal{\hat g_{t+1}}{A_t \hat g_{t+1}} - \scal{\tilde g_{t+1}}{A_t \tilde g_{t+1}} \,,
\]
which using that $P_{t+1} S_t^* Y_t = P_{t+1} A_t \tilde g_{t+1}$ by Equality~\eqref{eq:ptattildegt} yields
\begin{align*}
    \Omega(t+1) & = 2 \scal{\tilde g_{t+1} - \hat g_{t+1}}{P_{t+1} A_t \tilde g_{t+1}} + \scal{\hat g_{t+1}}{A_t \hat g_{t+1}} - \scal{\tilde g_{t+1}}{A_t \tilde g_{t+1}} \\ 
    & = 2 \scal{\tilde g_{t+1} - \hat g_{t+1}}{A_t \tilde g_{t+1}} + \scal{\hat g_{t+1}}{A_t \hat g_{t+1}} - \scal{\tilde g_{t+1}}{A_t \tilde g_{t+1}} \\ 
    & = -2\scal{\hat g_{t+1}}{A_t \tilde g_{t+1}} + \scal{\hat g_{t+1}}{A_t \hat g_{t+1}} +  \scal{\tilde g_{t+1}}{A_t \tilde g_{t+1}}  \\
    &= \scal{\tilde g_{t+1} - \hat g_{t+1}}{A_t (\tilde g_{t+1} - \hat g_{t+1})} \,.
\end{align*}
Let us denote $B_t = P_{t+1} A_t P_{t+1}$. Then, remarking that $\hat g_{t+1} = P_t \tilde A_t^{-1} P_t A_t \tilde g_{t+1}$ and  that $(P_{t+1} - P_t \tilde A_t^{-1} P_t A_t) P_t = 0$, we have
\begin{align}
    \Omega(t+1) &= \scal{(P_{t+1} - P_t \tilde A_t^{-1} P_t A_t) \tilde g_{t+1}}{A_t (P_{t+1} - P_t \tilde A_t^{-1} P_t A_t) \tilde g_{t+1}} \nonumber \\
    & = \scal{(P_{t+1} - P_t \tilde A_t^{-1} P_t B_t) \tilde g_{t+1}}{B_t (P_{t+1} - P_t \tilde A_t^{-1} P_t B_t) \tilde g_{t+1}} \nonumber \\
    &= \big\| B_t^{1/2} (P_{t+1} - P_t \tilde A_t^{-1} P_t B_t) \tilde g_{t+1} \big\|^2 \nonumber \\
    &= \big\| B_t^{1/2} (P_{t+1} - P_t \tilde A_t^{-1} P_t B_t) (P_{t+1} - P_t) \tilde g_{t+1} \big\|^2 \nonumber \\
    &\leq \big\| B_t^{1/2} (P_{t+1} - P_t \tilde A_t^{-1} P_t B_t)\|^2 \|(P_{t+1} - P_t) \tilde g_{t+1}\big\|^2  \,. 
    \label{eq:omega1}
\end{align}
We now upper-bound the two terms of the right-hand-side. For the first one, we use that  
\begin{align}
 \Big\| P_{t+1} & -  B_t^{1/2}  P_t \tilde A_t^{-1} P_t B_t^{1/2}  \Big\|^2  \nonumber \\
 & = \Big\|P_{t+1} - 2B_t^{1/2}  P_t \tilde A_t^{-1} P_t B_t^{1/2} +  B_t^{1/2}  P_t \tilde A_t^{-1} \underbrace{P_t B_t^{1/2}B_t^{1/2}  P_t \tilde A_t^{-1}}_{P_t} P_t B_t^{1/2} \Big\| \nonumber \\
 & = \left\| P_{t+1} -  B_t^{1/2}  P_t \tilde A_t^{-1} P_t B_t^{1/2} \right\|^2 \in \{0,1\} \label{eq:bt}\,,
\end{align}
where in the second equality we used that $P_t B_t^{1/2}B_t^{1/2}  P_t \tilde A_t^{-1} = P_t B_t  P_t \tilde A_t^{-1} = P_t \tilde A_t \tilde A_t^{-1} = P_t$.
Therefore, using that $B_t^{1/2}P_{t+1} = P_{t+1}B_t^{1/2}$ we get
\begin{align*}
    \| B_t^{1/2} (P_{t+1} - P_t \tilde A_t^{-1} P_t B_t) \|^2
    &=  \left\| B_t^{1/2} \left[ (P_{t+1} - P_t \tilde A_t^{-1} P_t B_t) (P_{t+1} - P_t) \right] \right\|^2 \\
    &= \left\| (P_{t+1}  - B_t^{1/2} s P_t \tilde A_t^{-1} P_t B_t^{1/2}) B_t^{1/2} (P_{t+1} - P_t) \right\|^2 \\
    &\leq \left\|  P_{t+1} -  B_t^{1/2}  P_t \tilde A_t^{-1} P_t B_t^{1/2} \right\|^2 \left\| B_t^{1/2} (P_{t+1} - P_t) \right\|^2 \\
    &\stackrel{\eqref{eq:bt}}{\leq} \left\| B_t^{1/2} (P_{t+1} - P_t) \right\|^2\\
    &\le \left\| C_t^{1/2} (P_{t+1} - P_t) \right\|^2 + \lambda \\
    &\le \mu_t + \lambda \,,
\end{align*}
where $\mu_t := \big\|(P_{t+1} - P_t)C_t^{1/2}\big\|^2$. Plugging back into~\eqref{eq:omega1}, this yields
\begin{equation} \label{eq:omega2}
    \Omega(t+1) \leq (\mu_t + \lambda) \|(P_{t+1} - P_t) \tilde g_{t+1}\|^2 \,.
\end{equation}
Then, substituting $\tilde g_{t+1}$ with its definition and using $\|Y_t\|^2 \leq tB^2$, we get
\begin{eqnarray*}
    \|(P_{t+1} - P_t) \tilde g_{t+1}\|^2
    & = &  \|(P_{t+1} - P_t) A_t^{-1} S_t^* Y_t \|^2 \\
    & \stackrel{\text{(Cauchy-Schwarz)}}{\le} &  \|(P_{t+1} - P_t) A_t^{-1} S_t^*\|^2 \|Y_t\|^2 \\
    & \le & t B^2 \|(P_{t+1} - P_t) A_t^{-1} S_t^* S_t A_t^{-1} (P_{t+1} - P_t)\| \\
    & \stackrel{(C_t = S_t^*S_t)}{=} & t B^2 \|(P_{t+1} - P_t) A_t^{-1} C_t A_t^{-1} (P_{t+1} - P_t)\|\,.
\end{eqnarray*}
Because $C_t$ and $A_t = C_t + \lambda I$ are co-diagonalizable, we have 
\begin{equation*}
    C_t^{1/2} A_t^{-1} = A_t^{-1} C_t^{1/2} \,,
\end{equation*}
which, together with $\|A_t^{-2}\| \leq 1/\lambda^2$ leads to 
\begin{eqnarray*}
 \|(P_{t+1} - P_t) \tilde g_{t+1}\|^2   & \leq  &  t B^2 \|(P_{t+1} - P_t) C_t^{1/2} A_t^{-2} C_t^{1/2} (P_{t+1} - P_t)\| \\
    &\le & \frac{t B^2}{\lambda^2} \|(P_{t+1} - P_t) C_t^{1/2} \|^2 \\
    &= & \frac{t \mu_t B^2}{\lambda^2} \,.
\end{eqnarray*}
Therefore, Inequality~\eqref{eq:omega2} concludes the proof of the second part
\begin{equation} \label{eq:part2}
    \Omega(t+1) \le (\mu_t + \lambda) \frac{t\mu_t B^2}{\lambda^2} \,.
\end{equation}

\paragraph{Conclusion of the proof.} Combining~\eqref{eq:defZ},~\eqref{eq:part1}, and~\eqref{eq:part2}, we obtain
\begin{eqnarray*}
    R_n(f) & \leq &  \sum_{t=1}^n y_t^2 \scal{\tilde A_t^{-1} P_t \phi(x_t)}{P_t \phi(x_t)}  + \sum_{t=1}^{n+1} \Omega(t) \\
        & \leq & \sum_{t=1}^n y_t^2 \scal{\tilde A_t^{-1} P_t \phi(x_t)}{P_t \phi(x_t)}  + \sum_{t=1}^{n+1} (\mu_{t-1} + \lambda) \frac{(t-1)\mu_{t-1} B^2}{\lambda^2} \,,
\end{eqnarray*}
which concludes the proof of the theorem.
\end{proof}

\section{Proofs of Section~\ref{sec:background} (\kawv{})}
\label{app:proofsbackground}

\subsection{Proof of Proposition~\ref{prop:nonLinearRidge}}
\label{app:kernelnonlinearRidge}

\medskip
First, remark that \kawv{} corresponds to \pkawv{} with $\tilde \cH_t = \cH$ and thus $P_t = I$ for all $t \geq 1$. Therefore, applying Theorem \ref{thm:mainLemma} with $P_t = I$ yields the regret bound,
\[ 
    R_n(f) \leq \lambda \|f\|^2 + \sum_{t=1}^n \scalh{A_t^{-1}\phi(x_t)}{\phi(x_t)} \,,
\]
for all $f \in \cH$. The rest of the proof consists in upper-bounding the second term in the right hand side. Remarking that $A_t = A_{t-1} + \phi(x_t) \otimes \phi(x_t)$ and applying Lemma~\ref{lem:linearalgebra} stated below we have
\[
    \scalh{A_t^{-1}\phi(x_t)}{\phi(x_t)} = 1 - \frac{\det(A_{t-1}/\lambda)}{\det(A_t/\lambda)} \,.
\]
It is worth pointing out that $\det(A_t/\lambda)$ is well defined since $A_t = I + C_t$ with  $C_t = \sum_{s=1}^t \phi(x_s)\otimes \phi(x_s)$ at most of rank $t\geq 0$. 
Then we use $1-u\leq \log (1/u)$ for $u>0$ which yields
\[
    \scalh{A_t^{-1}\phi(x_t)}{\phi(x_t)} \leq \log \frac{\det(A_{t}/\lambda)}{\det(A_{t-1}/\lambda)} \,.
\]
Summing over $t=1,\dots,n$, using $A_0 = \lambda I$ and $A_n = \lambda I + C_n$ we get
\begin{align*}
     \sum_{t=1}^n \scalh{A_t^{-1}\phi(x_t)}{\phi(x_t)} 
        & \leq \log \left(\det\Big(I + \frac{C_n}{\lambda}\Big)\right) \\
        & = \sum_{k=1}^\infty \log\left(1+\frac{\lambda_k(C_n)}{\lambda}\right) \,,
\end{align*}
which concludes the proof.

The following Lemma is a standard result of online matrix theory (see~Lemma~11.11 of \cite{Cesa-Bianchi2006}).

\begin{lemma}
    \label{lem:linearalgebra}
    Let $V:\cH \to \cH$ be a linear operator. Let $u \in \cH$ and let $U = V - u \otimes u$. Then,
    \[
        \scalh{V^{-1} u}{u} = 1 - \frac{\det(U)}{\det(V)} \,.
    \]
\end{lemma}

\subsection{Proof of \texorpdfstring{Proposition~\ref{prop:upperbounddeff}}{Proposition 2.2}}
Using that for $x >0$
\[
    \log(1+x) \leq \frac{x}{x+1} (1+\log(1+x)) \,,
\]
and denoting by $a(\la)$ the quantity
$a(s,\la) := 1 + \log (1+ s/\la)$,
we get for any $n\geq 1$
\[
    \log \! \Big(1 + \frac{\lambda_k(\Kn)}{\lambda}\Big) \leq \frac{\lambda_k(\Kn)}{\lambda + \lambda_k(\Kn)} a(\la_k(\Kn),\la).
\]
Therefore, summing over $k\geq 1$ and denoting by $\la_1$ the largest eigenvalue of $\Kn$
\begin{align}
    \sum_{k=1}^n \log\Big(1 + \frac{\lambda_k(\Kn)}{\lambda}\Big) 
        & \leq a(\la_1,\la) \sum_{k=1}^n  \frac{\lambda_k(\Kn)}{\lambda + \lambda_k(\Kn)} \label{eq:majDeff}\\
        & = a(\la_1,\la) \tr\big(\Kn(\Kn + \lambda I)^{-1}\big) \nonumber \\
        & = a(\la_1,\la)\deff(\lambda) \nonumber 
\end{align}
where the last equality is from the definition of $\deff(\lambda)$. 
Combining with Proposition~\ref{prop:nonLinearRidge}, substituting $a$ and upper-bounding $$\lambda_1(\Kn) \leq \tr(\Kn) = \sum_{t=1}^n \|\phi(x_t)\|^2 \leq  n \kappa^2$$  concludes the proof.

\section{Proofs of Section~\ref{sec:taylor} (\pkawv{} with Taylor's expansion)}
\label{app:proofTaylor}

\subsection{Proof of Theorem~\ref{thm:pkawv-appr-error}}
\label{sec:proofthm_pkawv-appr-error}

Applying Theorem~\ref{thm:mainLemma} with a fix projection $P$ and following the lines of the proof of Proposition \ref{prop:nonLinearRidge} we get
\[ 
    R_n(f) \leq \lambda \|f\|^2 + B^2 \sum\limits_{j=1}^n \log \left( 1 + \frac{\lambda_j(P C_n P)}{\lambda} \right) + (\mu + \lambda) \frac{n\mu B^2}{\lambda^2}  \,,
\]
where $\mu = \|(I - P) C_n^{1/2} \|^2$.
Moreover we have for all $i =1,\dots,n$ using that $\tilde C_n = PC_nP = PS_nS_n^*P$, we have
\[
\la_i(\tCn) = \la_i(P \Cn P) = \la_i(P \Sn^* \Sn P)  = \la_i(\Sn P P \Sn^* ) = \la_i(\Sn P \Sn^* )  \leq \la_i(K_{nn}).
\]

\subsection{Proof of Theorem~\ref{thm:phi-awv-gaussian}}
To apply our \cref{thm:pkawv-appr-error}, we need first (1) to recall that the functions $g_k$, $k \in \N_0^d$ are in $\hh$, (2) to show that they can approximate perfectly the kernel and (3) to quantify the approximation error of $G_M$ for the kernel function. First we recall some important existing results about the considered set of functions. For completeness, we provide self-contained (and often shorter and simplified) proofs of the following lemmas in Appendix~\ref{app:additional_proofs}.  

The next lemma states that $g_k$ with $k \in \N$ is an orthonormal basis for $\hh$ induced by the Gaussian kernel.
\blm[\cite{steinwart2006explicit}] \label{lm:orthonormal-basis}
For any $k, k' \in \N_0^d$, 
$$g_k \in \hh, \quad \|g_k\|_\hh = 1, \quad \scal{g_k}{g_{k'}}_\hh = \mathds{1}_{k = k'}\,.$$
\elm
Note that byproduct of the lemma, we have that $G_M \subset \hh$ and moreover that the matrix $Q$ is the identity, indeed $\smash{Q_{ij} = \scal{g_{k_i}}{g_{k_j}}_\hh = \mathds{1}_{k_i = k_j}}$. This means that the functions in $G_M$ are linearly independent. Moreover the fact that $Q = I_r$ further simplifies the computation of the embedding $\tphi$ (see~\eqref{eq:tphi}) in the implementation of the algorithm. 

The next lemma recalls the expansion of $k(x,x')$ in terms of the given basis.
\blm[\cite{cotter2011explicit}] \label{lem:kern-taylor-exp}
For any $x \in \X$, 
\eqal{\label{eq:kern-taylor-exp}
{k(x,x') = \scal{\phi(x)}{\phi(x')}_\hh = \sum_{k \in \N_0^d} g_k(x) g_k(x').}
}
\elm
Finally, next lemma provides approximation error of $k(x,x')$ in terms of the set of functions in $G_M$, when the data is contained in a ball or radius $R$. 
\blm[\cite{cotter2011explicit}]\label{lm:cotter}
Let $R > 0$. For any $x,x' \in \R^d$ such that $\|x\|, \|x'\| \leq R$ we have
\eqal{
{\Big|k(x,x') - \sum_{g \in G_M} g(x) g(x')\Big| \leq \frac{(R/\sigma)^{2M+2}}{(M+1)!}.}
}
\elm

Now we are ready to prove \cref{thm:phi-awv-gaussian}. 

\noindent{\bf Proof of point 1.} First, note that $r := |G_M|$, the cardinality of $G_M$, corresponds to the number of monomials of the polynomial $(1+x_1+\dots+x_d)^d$, i.e. $r := |G_M| =\binom{M + d}{M}$.
By recalling that $\binom{n}{k}\leq (en/k)^k$ for any $n, k \in \N$, we have
$$r = \binom{M + d}{M} = \binom{M + d}{d} \leq e^d(1 + M/d)^d.$$
We conclude the proof of the first point of the theorem, by considering that \pkawv{} used with the set of functions $G_M$ consists in running the online linear regression algorithm of~\cite{Vovk01, AzouryWarmuth2001} with $r := |G_M|$ features (see Appendix~\ref{app:implementation} for details). It
incurs thus a computational cost of $O(nr^2 + nrd)$ in time (no $r^3$ since we don't need to invert $Q$ which we have proven to be the identity matrix as consequence of \cref{lm:orthonormal-basis}) and $O(r^2)$ in memory.

\noindent{\bf Proof of point 2.} By \cref{lm:orthonormal-basis} we have that $G_M \subset \hh$ and $Q = I_r$, so the functions in $G_M$ are linearly independent. Then we can apply \cref{thm:pkawv-appr-error} obtaining the regret bound in \cref{eq:regret-appr-pawv}:
\begin{equation}
        R_n(f)  \leq \lambda \nor{f}^2 + B^2 \sum_{j=1}^n  \log\left(1+ \frac{\la_j(\Kn)}{\la} \right) 
        + B^2 \frac{(\mu + \lambda)n}{\lambda^2} \mu \,,
        \label{eq:regret-appr-pawv2}
\end{equation}
where $\mu := \big\|(I-P) C_n^{1/2}\big\|^2$ and $C_n := \sum_{t=1}^n \phi(x_t) \otimes \phi(x_t)$. The proof consists in upper-bounding the last approximation term $B^2 \frac{(\mu + \lambda)n}{\lambda^2} \mu$. We start by upper-bounding $\mu$ as follows
\begin{align*}
    \mu  :=  \big\|(I-P) C_n^{1/2}\big\|^2 
        & = \big\|(I-P) C_n (I-P)\big\| \\
        & = \left\|(I-P) \sum_{t=1}^n \phi(x_t) \otimes \phi(x_t) (I-P) \right\| \\
        & \leq \sum_{t=1}^n \left\|(I-P)  \phi(x_t) \otimes \phi(x_t) (I-P) \right\| \\
        & = \sum_{t=1}^n \left\|(I-P)\phi(x_t)\right\|^2 \\
        & = \sum_{t=1}^n \scal{(I-P)\phi(x_t)}{\phi(x_t)} \\
        & = \sum_{t=1}^n \scal{\phi(x_t)}{\phi(x_t)} - \scal{P\phi(x_t)}{P\phi(x_t)}  \\
        & = \sum_{t=1}^n k(x_t,x_t) - \|P \phi(x_t)\|^2 \,,
\end{align*}
where we used that $\scal{P \phi(x_t)}{\phi(x_t)} = \scal{P\phi(x_t)}{P\phi(x_t)}$. Now, since by Lemma~\ref{lm:orthonormal-basis}, the $g_k$ form an orthonormal basis of $\cH$, we have that
\[
    \|P \phi(x_t)\|^2 = \sum_{g \in G_M} g(x_t)^2 \,.
\]
where we recall that $P$ the projection onto $G_M$. Therefore, by Lemma~\ref{lm:cotter},
\begin{equation}
    \label{eq:upper-boundmu}
    \mu \leq \frac{(R/\sigma)^{2M+2} n }{(M+1)!} \quad  \stackrel{\text{Stirling}}{\leq} \quad   \frac{n e^{-(M+1) \log\big(  \frac{(M+1)\sigma^2}{eR^2}\big)}}{\sqrt{2\pi(M+1)}} 
    \leq \frac{n e^{-(M+1)}}{\sqrt{2\pi(M+1)}} \stackrel{M\geq 1}{\leq} \frac{n}{9}e^{-M} \,,
\end{equation}
where we used the fact that $n!$ is lower bounded by the Stirling approximation as $n! \geq \sqrt{2\pi n}e^{n\log\frac{n}{e}}$, for $n \in \N_0$ and that $M+1 \geq e^2 R^2/\sigma^2$, so $\log\frac{M+1}{e R^2/\sigma^2} \geq 1$. Now, since $M \geq 2\log(n/(\lambda \wedge 1))$, we have $M \geq \log(n/\lambda)$  and thus
\[
    \mu \leq \frac{n}{9}e^{-M} \leq \frac{\lambda}{9} \leq \lambda. 
\]
Therefore, the approximation term in~\eqref{eq:regret-appr-pawv2} is upper-bounded as
\[
   B^2 \frac{(\mu+\lambda)}{\lambda^2} \mu n \leq \frac{2B^2 \mu n}{\lambda} \stackrel{\eqref{eq:upper-boundmu}}{\leq} \frac{2B^2n^2 e^{-M}}{9\lambda}
\]
which using again $M \geq 2 \log (n/(\lambda \wedge 1))$ entails
\begin{equation}
    \label{eq:approxGaussian}
    B^2 \frac{(\mu+\lambda)}{\lambda^2} \mu n  \leq 
    \frac{2}{9} B^2 (\lambda \wedge \lambda^{-1}) \leq \frac{4B^2}{9} \log\Big(1+\frac{1}{\lambda}\Big) \,,
\end{equation}
where in the last inequality we used that $(\la \wedge \la^{-1})/2 \leq \log(1+ 1/\la)$ for any $\la > 0$.
Now, since $\log(1+x)$ is concave on $[0,\infty)$, by subadditivity 
$$ \textstyle{\sum_{j=1}^n \log\Big(1+\frac{\la_j(\Kn)}{\la}\Big) \geq  \log\Big(1+\sum_{j=1}^n\frac{\la_j(\Kn)}{\la}\Big).}$$
By definition of trace in terms of eigenvalues and of the diagonal of $\Kn$, we have 
$$\sum_{j=1}^n \la_j(\Kn) = \tr(\Kn) = \sum_{j=1}^n k(x_j, x_j) = n,$$ 
where the last step is due to the fact that for the Gaussian kernel we have $k(x,x) = 1$, for any $x \in \X$. Then
\eqal{
B^2\log\left(1 + \frac{1}{\la}\right) \leq B^2\log\left(1 + \frac{n}{\la}\right)
\leq {B^2 \sum_{j=1}^n \log\left(1 + \frac{\la_j(\Kn)}{\la}\right).}
}
Plugging back into Inequality~\eqref{eq:approxGaussian} and substituting into~\eqref{eq:regret-appr-pawv2} concludes the proof of the Theorem.

\section{Proofs of Section~\ref{sec:nystrom} (\pkawv{} with Nyström projections)}
\label{app:proof_Nystrom}

\subsection{Proof of Theorem~\ref{thm:pkawv-nystrom}}
\label{sec:proof_pkawv-nystrom}
The proof consists of a straightforward combination of Proposition~\ref{thm:KORS} and  Theorem~\ref{thm:mainLemma}. According to Proposition~\ref{thm:KORS}, with  probability at least $1-\delta$, we have for all $t\geq 1$, 
\[
    \mu_t = \|(P_{t+1} - P_t)C_t^{1/2}\|^2 \leq \|(I - P_t)C_t^{1/2}\|^2 \indic_{P_{t+1} \neq P_t} \leq \mu \indic_{P_{t+1} \neq P_t}\,,
\]
with $|\cI_n| \leq 9\deff(\mu) \log(2n/\delta)^2$. 
Therefore, from Theorem~\ref{thm:mainLemma}, if $\mu \leq \lambda$, the regret is upper-bounded as
\[ 
    R_n(f) \leq \lambda \|f\|^2 + B^2 \deff(\lambda) \log \left( e + \frac{en \kappa^2}{\lambda} \right) + 2\frac{\mu n (|\cI_n| + 1) B^2}{\lambda}  \,.
\]
Furthermore, similarly to any online linear regression algorithm in a $m$-dimensional space, the efficient implementation of the algorithm (see Appendix~\ref{app:implementation}) requires  $O(m^2)$ space and time per iteration, where $m = |\cI_n|$ is the size of the dictionary. This concludes the proof of the theorem.  

\subsection{Proof of Corollary~\ref{cor:krr_nystrom_rate}}
We recall that the notation $\lesssim$ denotes a rough inequality which is up to logarithmic multiplicative terms and may depend on unexplained constants. Here, we only consider non-constant quantities $n$, $\lambda$, $m$ and $\mu$ and focus on the polynomial dependence on $n$. Keeping this in mind, the high-probability regret upper-bound provided by Theorem~\ref{thm:pkawv-nystrom} can be rewritten as
\begin{equation}
    R_n(f) \lesssim \lambda + \left(\frac{n}{\lambda}\right)^\gamma + \frac{\mu n |\cI_n|}{\lambda} \,,
    \label{eq:regretratenystrom}
\end{equation}
for all $f \in \cH$. It only remains to optimize the parameters $\mu$ and $\lambda$. Choosing $\mu = \deff^{-1}(m)$ ensures that the size of the dictionary is upper-bounded as $|\cI_n| \lesssim \deff(\mu) = m$. 

Moreover, by assumption $m  = \deff(\mu) \leq \left( \frac{n}{\mu} \right)^\gamma$ and thus $\mu \leq nm^{-\frac{1}{\gamma}}$. Therefore, the regret is upper-bounded with high-probability as
\begin{equation}
    R_n(f) \lesssim \lambda + \left(\frac{n}{\lambda}\right)^\gamma + \frac{n^2 (m^{\frac{\gamma - 1}{\gamma}}+1)}{\lambda} \,.
    \label{eq:regretrate}
\end{equation}
Now, according to the value of $m$, two regimes are possible:
\begin{itemize}
    \item If the dictionary is large enough, i.e., $m \ge n^{\frac{2\gamma}{1-\gamma^2}}$ then, once $\lambda$ is optimized, the last term of the right-hand side is negligible. The regret upper-bound consists then in optimizing $\lambda + (n/\lambda)^\gamma$ in $\lambda$ yielding to the choice $\lambda = n^\frac{\gamma}{1+\gamma}$. We get the upper-bound
    \[
    R_n(f) \lesssim n^{\frac{\gamma}{\gamma + 1}} + n^\gamma n^{-\frac{\gamma^2}{1+\gamma}} + n^2 n^{-\frac{\gamma}{\gamma+1}} n^{\frac{-2}{\gamma+1}} 
    \lesssim n^{\frac{\gamma}{\gamma + 1}} \,,
    \]
    which recovers the optimal rate in this case.
    \item Otherwise, if $m \le n^{\frac{2\gamma}{1-\gamma^2}}$, then the last term of the r.h.s. of~\eqref{eq:regretrate} is predominant. The dictionary is too small to recover the optimal regret bound. The parameter $\lambda$ is optimizes the trade-off $\lambda + n^2 m^{(\gamma-1)/\gamma}/\lambda$ which leads to the choice $\lambda = nm^{\frac{1}{2} -\frac{1}{2\gamma}}$. The upper-bound on the regret is then
    \[
    R_n(f) \lesssim nm^{\frac{\gamma-1}{2\gamma}} + m^{\frac{1-\gamma}{2}} + n m^{\frac{1-\gamma}{2\gamma} + \frac{\gamma-1}{\gamma}} \lesssim nm^{\frac{\gamma-1}{2\gamma}} \,.
    \]
\end{itemize} 
This concludes the proof.

\subsection{Proof of Corollary \ref{cor:krr_nystrom_known_features_rate}}
The proof follows the lines of the one of Theorem~\ref{thm:pkawv-nystrom} and Corollary~\ref{cor:krr_nystrom_rate}. However, here since the projections are fixed we can apply Theorem~\ref{thm:pkawv-appr-error} instead of Theorem~\ref{thm:mainLemma}. This yields the high-probability regret upper-bound 
\[ 
    R_n(f) \lesssim \lambda + \left(\frac{n}{\lambda}\right)^\gamma + \frac{\mu n}{\lambda} \,,
\]
which improves by a factor $|\cI_n|$ the last term of the bound~\eqref{eq:regretratenystrom}. The choice $\mu = \deff^{-1}(m)$ yields with high probability $|\cI_n| \lesssim \deff(\mu) = m$ and $\mu \leq nm^{-\frac{1}{\gamma}}$ which entails
\[ 
    R_n(f) \lesssim \lambda + \left(\frac{n}{\lambda}\right)^\gamma + \frac{n^2 m^{-\frac{1}{\gamma}}}{\lambda} \,.
\]

Similarly to Corollary~\ref{cor:krr_nystrom_rate} two regimes are possible. The size of the dictionary decides which term is preponderant in the above upper-bound:
\begin{itemize}
    \item If $m \ge n^{\frac{2\gamma}{1+\gamma}}$ the dictionary is large enough to recover the optimal rate for the choice $\lambda = n^{\frac{\gamma}{1+\gamma}}$. Indeed it yields
    \[
    R_n(f) \lesssim n^{\frac{\gamma}{\gamma + 1}} + n^\gamma n^{-\frac{\gamma^2}{1+\gamma}} + n^{\frac{2\gamma}{\gamma+1}} n^{-\frac{\gamma}{\gamma+1}} \lesssim n^{\frac{\gamma}{\gamma + 1}}
    \]
    \item Otherwise $m \leq n^{\frac{2\gamma}{1+\gamma}}$ and the choice  $\lambda = n^{\frac{\gamma}{1+\gamma}}$ leads to 
    \[
    R_n(f) \lesssim
    n^{\frac{\gamma}{\gamma + 1}} + n^\gamma n^{-\frac{\gamma^2}{1+\gamma}} + n^{\frac{2\gamma}{\gamma+1}} n^{-\frac{\gamma}{\gamma+1}} 
    \lesssim n^{\frac{\gamma}{\gamma + 1}} \,.
    \]
    The last inequality is due to $m^{\frac{1}{2}} \le nm^{-\frac{1}{\gamma} + \frac{1}{2\gamma}}$ because $\gamma \le 1$ and $m \le n$.
\end{itemize}

\section{Proofs of additional lemmas}
\label{app:additional_proofs}
\subsection{Proof of \texorpdfstring{\cref{lm:orthonormal-basis}}{Lemma 4.2}}
Recall the following characterization of scalar product for translation invariant kernels (i.e. $k(x,x') = v(x-x')$ for a $v:\R^d \to \R$) \citep[see][]{berlinet2011reproducing}
$$\scal{f}{g}_\hh = \int \frac{{\cal F}[f](\omega){\cal F}[g](\omega)}{{\cal F}[v](\omega)},$$
where ${\cal F}[f]$ is the unitary Fourier transform of $f$. 
Let start from the one dimensional case and denote by $\hh_0$ the Gaussian RKHS on $\R$. First note that when $d=1$, we have $g_k = \psi_k$. 
Now, the Fourier transform of $\psi_k$ is ${\cal F}[\psi_k](\omega) = \frac{1}{\sqrt{k!}} H_k(x/\sigma^2) e^{-\omega^2/(2\sigma^2)}$, for any $k \in \N_0^d$, where $H_k(x)$ is the $k$-th Hermite polynomial \citep[see][Eq.~18.17.35 pag. 457]{olver2010nist}, and ${\cal F}[v] = e^{-\omega^2/2}$, then, by the fact that Hermite are orthogonal polynomial with respect to $e^{-\omega^2/2}$ forming a complete basis, we have
$$\scal{\psi_k}{\psi_{k'}}_{\hh_0} = \frac{1}{k!}\int H_k(\omega) H_{k'}(\omega) e^{-\omega^2/2} d\omega = \mathds{1}_{k = k'}.$$
The multidimensional case is straightforward since Gaussian is a product kernel, i.e.
$k(x,x') = \prod_{i=1}^d k(x^{(i)},x^{(i)})$ and $\hh = \otimes_{i=1}^d \hh_0$, so $\scal{\otimes_{i=1}^d f_i}{\otimes_{i=1}^d g_i}_{\hh} = \prod_{i=1}^d \scal{f_i}{g_i}_{\hh_0}$ \citep[see][]{aronszajn1950theory}. Now, since $g_k = \otimes_{i=1}^d  \psi_{k_i}$, we have $\scal{g_k}{g_{k'}}_{\hh} = \prod_{i=1}^d \scal{\psi_{k_i}}{\psi_{k_i'}}_{\hh_0} = \mathds{1}_{k = k'}.$

\subsection{Proof of \texorpdfstring{\cref{lem:kern-taylor-exp}}{Lemma 4.3}}

First, for $j \in \N_0$ define
$$Q_j(x,x') := e^{-\frac{\|x\|^2}{2\sigma^2}-\frac{\|x'\|^2}{2\sigma^2}} \frac{(x^\top x'/\sigma^2)^j}{j!}.$$
First note that, by multinomial expansion of $(x^\top x')^{j}$,
\eqals{
Q_j(x, x') &= \frac{e^{-\frac{\|x\|^2 + \|x'\|^2}{2\sigma^2}}}{\sigma^{2j}j!} \sum_{|t| = j} \binom{j}{t_1 \dots t_d} \prod_{i=1}^d (x^{(i)})^{t_i} ({x'}^{(i)})^{t_i} \\
& = \sum_{|t|=j} g_t(x) g_t(x').
}
Now note that, by Taylor expansion of $e^{x^\top x'/\sigma^2}$ we have
\eqals{
k(x,x') &= \sum_{j=0}^\infty Q_j(x,x') = \sum_{j=0}^\infty \sum_{|t|=j} g_t(x)g_t(x') \\
& = \sum_{k \in \N_0^d}g_k(x) g_k(x').
}
Finally, with $\phi$ defined as above, and the fact that $g_k$ forms an orthonormal basis for $\hh$, leads to
$$\scal{\phi(x)}{\phi'(x)}  =  \sum_{k \in \N_0^d}g_k(x) g_{k}(x') = k(x,x').$$

\subsection{Proof of \texorpdfstring{\cref{lm:cotter}}{Lemma 4.4}}
Here we use the same notation of the proof of Lemma~\ref{lem:kern-taylor-exp}. 
Since by Taylor expansion, we have that $k(x,x') = \sum_{j=0}^\infty Q_j(x,x')$,
by mean value theorem for the function $f(s) = e^{s/\sigma^2}$, we have that there exists $c \in [0, x^\top x']$ such that
\eqals{
|k(x,x') - \sum_{j=0}^M Q_j(x,x')|
 &= e^{-\frac{\|x\|^2+\|x'\|^2}{2\sigma^2}} \frac{c^{M+1}}{(M+1)!} \frac{d^{M+1}e^{\frac{s}{\sigma^2}}}{ds^{M+1}}|_{s=c} \\ 
& \leq \frac{(|x^\top x'|/\sigma^2)^{M+1}}{(M+1)!} \\
&\leq \frac{(R/\sigma)^{2M + 2}}{(M+1)!} \\
}
where the last step is obtained assuming $\|x\|,\|x'\| \leq R$.
Finally note that, by definition of $G_M$,
$$\sum_{g \in G_M} g(x) g(x') = \sum_{|k| \leq M}^M g_k(x) g_{k}(x') = \sum_{j=0}^M Q_j(x,x').$$

\section{Additional experiments}
\label{app:other_experiments}

 \begin{figure*}[t]
\centering
\includegraphics[scale=0.3]{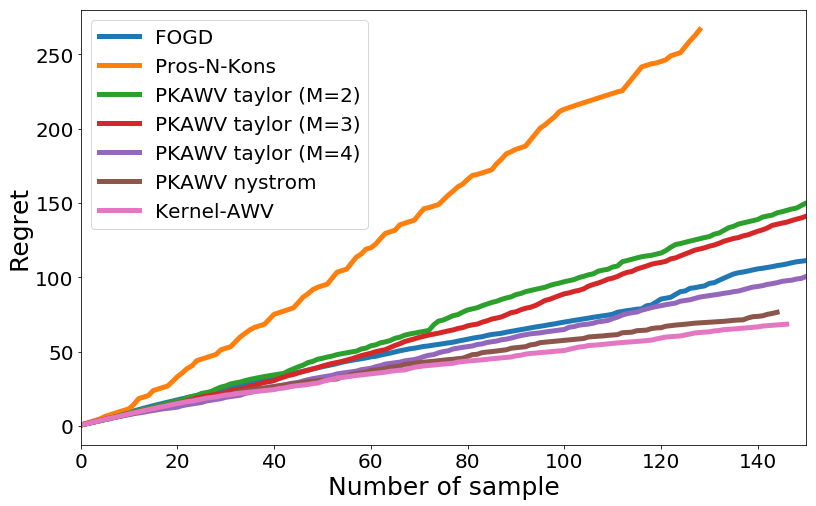}
\caption{Regret in adversarial setting.}
\label{fig:adv_regret}
\end{figure*}

 \begin{figure*}[t]
\centering
\includegraphics[scale=.25]{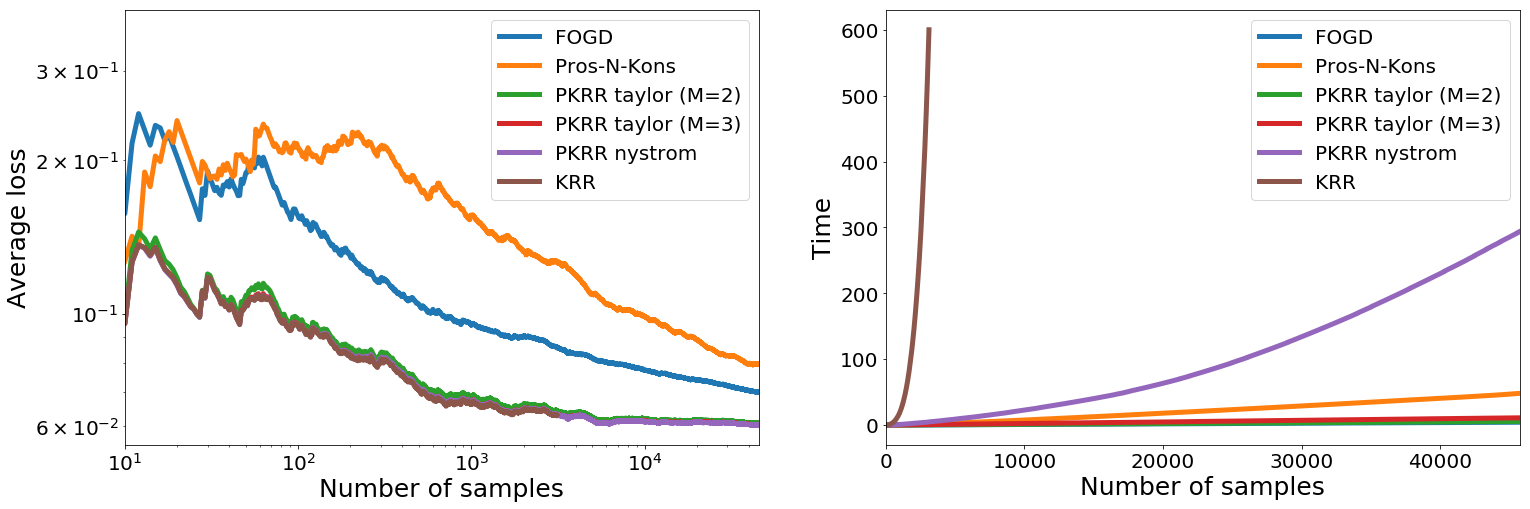}
\centering
\includegraphics[scale=.25]{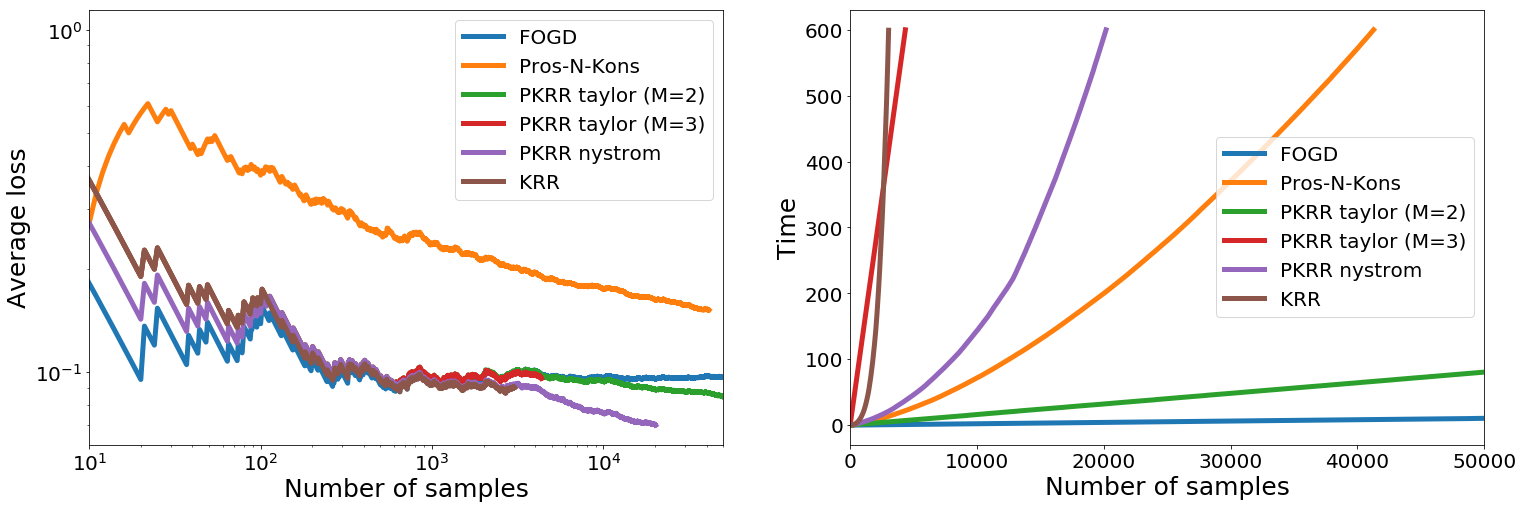}
\caption{Average loss and time on (top): regression casp ($n=4.5\times10^4,~~d=9$); (bottom) classification \texttt{ijcnn1} ($n=1.5\times10^5,~~d=22$).}
\label{casp_perf}\label{ijcnn1_perf}
\end{figure*}

\noindent{\bf Additional large scale datasets} (cf. Figure \ref{casp_perf}).
\label{sec:other_dataset}
We provides results on two additional datasets from UCI machine learning repository : \texttt{casp} (regression) and \texttt{ijcnn1}. See section \ref{sec:experiments} for more details.

\noindent{\bf Adversarial simulated data} (cf. Figure \ref{fig:adv_regret})
\label{sec:exp-adv}
In this experiment we produced the sequence $(x_t, y_t)_{t \in \N}$ adversarially on the regret function. In particular, given the learning algorithm, we use \texttt{scipy} as a greedy adversary i.e. at each step an optimization is done on the regret function to find $(x_t,y_t)$. On the right of Figure~\ref{fig:adv_regret}, we plot the  simulations until $n=80$, with $(x_t,y_t) \in [-1,1]^d \times [-1,1]$ where $d=5$. We see that \kawv{}, which does not use any approximation, leads to the best regret. Furthermore, \pkawv{} approximations converges very fast to the regret of \kawv{} when $M$ increases. The poor performance of Pros-N-Kons is likely because of its frequent restarts which is harmful when $n$ is small. On the contrary, FOGD has surprisingly good performance. We run the simulations up to $n=80$ for the high computational cost required by the adversary (especially for algorithms like \kawv{} or Pros-N-Kons).

\section{Efficient implementation of \pkawv{}}
\label{app:implementation}

\subsection{Pseudo-code}

Here, we detail how the formula~\eqref{eq:pkawv-def} can be efficiently computed for the projections considered in Section~\ref{sec:kernelproj}.

\paragraph{Fixed embedding} We consider fix sub-spaces $\tilde \cH_t = \tilde \cH$ induced fixed by the span of a fixed set of functions $G = \{g_1,\dots,g_r\} \subset \cH$ as analyzed in Section~\ref{sec:taylor}. Let denote by $\tphi:\X \to \R^r$ the map 
\eqal{\label{eq:tphi}
\tphi(x) &= Q^{-1/2} v(x),
}
with $v(x) = (g_1(x),\dots, g_r(x))$, and $Q \in \R^{r\times r}$ defined as $Q_{ij} = \scal{g_i}{g_j}_\hh$.
Then, computing the prediction $\hat y_t = \hat f_t(x_t)$ of \pkawv{} with
\[
    \hat f_t  \in  \argmin{f \in \tilde \cH = \mathrm{Span}(G)} \left\{ \sum_{s=1}^{t-1} \big(y_s - f(x_s)\big)^2 + \lambda \nor{f}^2 + f(x_t)^2 \right\}
\]
is equivalent to embedding $x_t$ in $\R^r$ via $\tphi$ and then performing linear AWV of \cite{AzouryWarmuth2001,Vovk01} with $\hat y_t = \hat w_t^\top \tphi(x_t)$
\[
    \hat w_t \in \argmin{w \in \R^r} \left\{ \sum_{s=1}^{t-1} \big(y_s - w^\top \tphi(x_s)\big)^2 + \lambda \nor{w}^2 + \big(w^\top \tphi(x_t)\big)^2 \right\} \,.
\]
This reduces the total computational complexity to $O(n r^2 + n r d + r^3)$ in time and $O(r^2)$ in space (see Algorithm~\ref{alg:pkrrfixed} for an efficient implementation).

\begin{algorithm}
{\bfseries Input}: $\lambda >0$, $\tilde \phi: \cX \to \R^r$ for $r \geq 1$ \\[2pt]
    {\bfseries Initialization}: $A_0^{-1} =\lambda^{-1} I_r$, $b_0 = 0$ \\[2pt]
    {\bfseries For} $t=1,\dots, n$
    \begin{itemize}[topsep=5pt,label=--,leftmargin=25pt,parsep=2pt]
        \item receive $x_t \in \cX$ 
        \item compute $v_t = \tphi(x_t) \in \R^r$
        \item update $A_t^{-1} = A_{t-1}^{-1} - \frac{(A_t^{-1} v_t)(A_t^{-1} v_t)^\top}{1+v_t^\top A_t^{-1}v_t}$
        \item predict $\hat y_t = \tilde \phi(x_t)^\top A_t^{-1} b_{t-1}$
        \item receive $y_t \in \R$
        \item update $b_t = b_{t-1} + v_t y_t$
    \end{itemize}
\caption{\pkawv with fixed embedding}
\label{alg:pkrrfixed}
\end{algorithm}

\paragraph{Nyström projections}
Here, we detail how our algorithm can be efficiently implemented with Nyström projections as considered in section~\ref{sec:nystrom}. If we implement naïvely this algorithm, we would compute $\alpha_t = (K_{t,m_t}^T K_{t,m_t} + \lambda K_{m_t, m_t})^{-1} K_{t,m_t}^T Y_t$ at each iteration. However, it would require $n \deff(\mu) + \deff(\mu)^3$ operations per iterations. We could have update this inverse with Sherman–Morrison formula and Woodbury formula. However, in practice it leads to numeric instability because the matrix can have small eigenvalues. Here we use a method described in \cite{rudi2015less}. The idea is to use the cholesky decomposition and  cholup which update the cholesky decomposition when adding a rank one matrix i.e. if $A_t = L_t^T L_t$ and $A_{t+1} = A_t + u_{t+1}u_{t+1}^T$ then $L_{t+1} = \textrm{cholup}(L_t, u_{t+1}, \textrm{'+'})$. Updating the cholesky decomposition with cholup require only $\deff(\mu)^2$ operations. So, \pkawv{} with nyström has a $O(n \deff(\mu) + \deff(\mu)^2)$ time complexity per iterations.

{\color{red}
\begin{algorithm}
{\bfseries Input}: $\lambda, \mu, \beta >0$, \\[2pt]
    {\bfseries Initialization}: $d_1 = $ \\[2pt]
    {\bfseries For} $t=1,\dots, n$
    \begin{itemize}[topsep=5pt,label=--,leftmargin=25pt,parsep=2pt]
        \item receive $x_t \in \cX$ 
        \item compute $z_t$ with KORS
        \item $K_t = (k(x_i, \tilde x_j))_{i\leq t, j \in \mathcal{I}_{t-1}}$
        \item $\mathcal{I}_t = \mathcal{I}_{t-1}$
        \item $a_t = (k(x_t,x_1),...,k(x_t, x_t))$
        \item $R_t = \textrm{cholup}(R_t, a_t, \textrm{'+'})$ \\[2pt]
        {\bfseries If} $z_t = 1$
        \begin{itemize}[topsep=5pt,label=--,leftmargin=25pt,parsep=2pt]
            \item $\mathcal{I}_t = \mathcal{I}_t \cup \{t\}$
            \item $K_t = (k(x_i, x_j))_{i\leq t, j \in \mathcal{I}_{t}}$
            \item $b_t = (k(x_t, x_j))_{j \in \mathcal{I}_{t}}$
            \item $c_t = K_{t-1}^T a_t + \lambda b_t$
            \item $d_t = a_t^T a_t + \lambda k(x_t, x_t)$
            \item $g_t = \sqrt{1 + d_t}$
            \item $u_t = (c_t/(1+g_t), g_t)$
            \item $v_t = (c_t/(1+g_t), -1)$
            \item $R_t = \left( 
            \begin{array}{cc}
                R_{t-1} & 0  \\
                0 & 0 \\
            \end{array}\right)$
            \item $R_t = \textrm{cholup}(R_t, u_t, \textrm{'+'})$
            \item $R_t = \textrm{cholup}(R_t, v_t, \textrm{'-'})$
        \end{itemize}
        \item $\alpha_t = R_t^{-1} R_t^{-T} K_t^T (Y_t, 0)$
        \item $b_t = (k(x_t, x_j))_{j \in \mathcal{I}_{t}}$
        \item predict $\hat y_t = b_t^T \alpha_t$
        \item receive $y_t \in \mathbb{R}$
        \item update $Y_t = (Y_{t-1}, y_t)$
    \end{itemize}
\caption{\pkawv with Nyström projections}
\label{alg:pkrrnystrom}
\end{algorithm}}

\newpage
\clearpage
\subsection{Python code}\label{app:python-code}

{
\footnotesize
\begin{lstlisting}
import numpy as np
from math import factorial

class PhiAWV:
   def __init__(self, d, sigma=1.0, lbd=1.0, M=2):
       self.b = None
       self.A_inv = None
       self.M = M
       self.lbd = lbd
       self.sigma = sigma
       
   def taylor_phi(self, x):
       res = np.array([1.])
       mm = np.array([1.])
       for k in range(1,self.M+1):
           mm = (np.outer(mm, x)).flatten()
           q = self.sigma**k*np.sqrt(factorial(k))
           res = np.concatenate((res,mm/q))
       c  = 2*self.sigma**2
       res *= np.exp(-np.dot(x,x)/c)
       return np.array(res)

   def predict(self, x):
       z = self.taylor_phi(x)
       if self.b is None:
           r = len(z)
           self.b = np.zeros(r)
           self.A_inv = (1/self.lbd)*np.eye(r)

       v = np.dot(self.A_inv, z)
       v /= np.sqrt(1 + np.dot(z, v))
       self.A_inv -= np.outer(v, v)
       w_hat = np.dot(self.A_inv, self.b)
       self.z = z
       return np.dot(w_hat, z)

   def update(self, y):
       self.b += y*self.z

def update_inv(A_inv,x):
    B = x[:-1][:,None]
    C = np.transpose(B)
    D = np.array(x[-1])[None,None]
    if A_inv.size == 0:
        return 1./D
    compl = 1./(D - np.dot(np.dot(C,A_inv),B))
    R0 = A_inv + np.dot(np.dot(np.dot(np.dot(A_inv,B),compl),C),A_inv)
    R1 = - np.dot(np.dot(A_inv,B),compl)
    R2 = - np.dot(np.dot(compl,C),A_inv)
    R3 = np.array(compl)
    return np.block([[R0, R1], [R2, R3]])

def KORS(x, KMM, S, SKS_inv, lbd=1., eps=0.5, beta=1.):
    kS = KMM[:,-1]*S
    SKS = np.array(S)[None,:] * KMM * np.array(S)[:,None]
    en = np.eye(len(kS))[:,-1]
    SKS_inv_tmp = update_inv(SKS_inv, kS + lbd*en)
    tau = (1+eps)/lbd*(KMM[-1,-1] - np.dot(kS, np.dot(SKS_inv_tmp, kS)))
    p = max(min(beta*tau,1),0)
    z = np.random.binomial(1,p)
    S = S[:-1]
    if z:
        S.append(1/p)
        SKS_inv = update_inv(SKS_inv, 1/p*KMM[:,-1]*S + lbd*en)
    return z, S, SKS_inv

class Nystrom_kernel_AWV:
    def __init__(self, d, k, lbd=1.):
        self.c = np.zeros(0)
        self.X = np.zeros((0,d))
        self.Y = np.zeros(0)
        self.k = k
        self.lbd = lbd
        self.R = np.eye(0)
        self.chosen_idx = []
        self.KnM = np.eye(0)
        self.S = []
        self.SKS_inv = np.eye(0)

    def predict(self, x):
        self.X = np.concatenate((self.X,x[None,:]), axis=0)
        n = self.X.shape[0]
        Kn = np.array([self.k(self.X[i,:],x) for i in self.chosen_idx])
        self.KnM = np.concatenate((self.KnM,Kn[None,:]), axis=0)
        K_kors = np.concatenate((self.KnM[self.chosen_idx+[n-1],:], 
                                 np.concatenate((Kn, [self.k(x, x)]))[:,None]),
                                                axis=1)
        z, self.S, self.SKS_inv = KORS(x, K_kors, list(self.S)+[1], \
                                       self.SKS_inv, lbd=self.lbd)
        self.R = cholup(self.R, self.KnM[-1,:], '+')
        if z:
            self.chosen_idx.append(n-1)
            M = len(self.chosen_idx)
            KM = np.array([self.k(self.X[i,:],x) for i in range(n)])
            self.KnM = np.concatenate((self.KnM,KM[:,None]), axis=1)
            a = self.KnM[:,-1].T
            d = np.dot(a, a) + self.lbd*self.KnM[-1,-1]
            if M == 1:
                self.R = np.array([[np.sqrt(d)]])
            else:
                b = self.KnM[self.chosen_idx[:-1],-1]
                c = np.dot(self.KnM[:,:-1].T, a) + self.lbd*b
                g = np.sqrt(1 + d)
                u = np.concatenate((c/(1+g), [g]))
                v = np.concatenate((c/(1+g), [-1]))
                self.R = np.block([[self.R, np.zeros((M-1,1))],
                                  [np.zeros((1,M-1)), 0]])
                self.R = cholup(self.R, u, '+')
                self.R = cholup(self.R, v, '-')
        Yp = np.concatenate((self.Y, [0]))
        if len(self.R) > 0:
            self.c = solve_triangular(self.R, 
                     solve_triangular(self.R.T, np.dot(self.KnM.T,self.Y),
                                      lower=True))
        Kn = np.array([self.k(self.X[i,:],x) for i in self.chosen_idx])
        return np.dot(Kn, self.c)

    def update(self, y):
        self.Y = np.concatenate((self.Y,np.array(y)[None]), axis=0)

\end{lstlisting}}

\end{document}

%% file: macro.tex
\newcommand{\Nystrom}[1]{{Nystr\"om}}

\providecommand{\scal}[2]{\left\langle{#1},{#2}\right\rangle}
\providecommand{\nor}[1]{\bigl\|{#1}\bigr\|}

\providecommand{\tr}{\operatorname{Tr}}

\newcommand{\R}{\mathbb R}

\newcommand{\N}{\mathbb N}

\newcommand{\hh}{\mathcal H}

\newcommand{\la}{\lambda}

\newcommand{\argmin}[1]{\mathop{\operatorname{argmin}}_{#1}}

\newcommand{\X}{{\cal X}}

\newcommand{\C}{{C}}

\newcommand{\Cn}{{\C}_n}

\newcommand{\Sn}{S_n}

\providecommand{\scalh}[2]{\left\langle{#1},{#2}\right\rangle_\hh}

%% file: rates.tex
    \large 
    \pgfmathsetmacro\yopt{\gam/(1+\gam)}
    \pgfmathsetmacro\yunopt{4*\gam/(1+\gam)^2}
    \pgfmathsetmacro\xx{2*\gam/(1+\gam)}
    \pgfmathsetmacro\xxx{2*\gam/(1-\gam^2)}
    
    \begin{axis}%
    [
        xmin=0,
        xmax=1.4,
        axis x line=bottom,
        line width=.3mm,
        axis y line=left,
        ytick={0,\yopt,\yunopt,1},
        yticklabels={0,$\frac{\gamma}{1+\gamma}$,$\frac{4\gamma}{(1+\gamma)^2}$,1},
        xtick={0,\xx,\xxx,1},
        xticklabels={0,$\frac{2\gamma}{1+\gamma}$,$\frac{2\gamma}{1-\gamma^2}$,1},
        ymin=0,
        ymax=1.4,
        ylabel=$b$,
        xlabel=$a$,
        restrict x to domain=0:1,
        legend style={draw=none},
        legend cell align={left}
    ]
        \addplot%
        [
            red,%
            dashed,
            line width=.3mm,
            mark=none,
            samples=100,
            domain=0:1,
        ]
        (x,{1-x/(1+\gam)});

        \addplot%
        [
            orange,%
            dashdotted,
            line width=.3mm,
            mark=none,
            samples=100,
            domain=0:(2*\gam*(1-\gam)/(1+\gam)^2),
        ]
        (x,{1+(\gam-1)*x/(2*\gam)});

        \addplot%
        [
            blue,%
            mark=none,
            line width=.3mm,
            samples=100,
            domain=0:(2*\gam/(1-\gam^2)),
        ]
        (x,{1+(\gam-1)*x/(2*\gam)});
         
        \addplot%
        [
            black,%
            dotted,
            line width=.4mm,
            mark=none,
            samples=100,
            domain=0:(2*\gam/(1+\gam)),
        ]
        (x,{1-x/(2*\gam)});

        \addplot%
        [
            blue,%
            mark=none,
            line width=.3mm,
            samples=100,
            domain=(2*\gam/(1-\gam^2)):1,
        ]
        (x,{\gam/(1+\gam)}); 

        \addplot%
        [
            orange,%
            dashdotted,
            line width=.3mm,
            mark=none,
            samples=100,
            domain=(2*\gam*(1-\gam)/(1+\gam)^2):1,
        ]
        (x,{1-(\gam-1)^2/(1+\gam)^2});

        \addplot%
        [
            black,%
            dotted,
            line width=.4mm,
            mark=none,
            samples=100,
            domain=(2*\gam/(1+\gam)):1,
        ]
        (x,{\gam/(1+\gam)});
    
    \normalsize
    \ifthenelse{\equal{\plotlegend}{T}}{
    \addlegendentry{Sketched-KONS~\cite{calandriello2017second}}
	\addlegendentry{Pros-N-KONS~\cite{calandriello2017efficient}}
	\addlegendentry{\pkawv{}}
	\addlegendentry{\pkawv{} (beforehand features)}}{}
\end{axis}